\documentclass[10pt,journal,compsoc]{IEEEtran}
\ifCLASSOPTIONcompsoc
  \usepackage[nocompress]{cite}
\else
  \usepackage{cite}
\fi
\ifCLASSINFOpdf
\else
\fi

\usepackage[colorlinks=false]{hyperref}
\usepackage{times}
\usepackage{graphicx}
\usepackage{amsmath}
\usepackage{amssymb}
\usepackage{subfig}
\usepackage{booktabs}
\usepackage{epstopdf}
\usepackage{algorithm}
\usepackage{algpseudocode}
\usepackage{algpascal}
\usepackage{float}
\usepackage{multirow}
\newtheorem{theorem}{Theorem}

\newtheorem{proof}{Proof}
 
\begin{document}
%
\title{Deep Distribution-preserving Incomplete Clustering with Optimal Transport}

\author{Mingjie~Luo, ~Siwei~Wang, ~Xinwang~Liu, ~Wenxuan~Tu,  ~Yi~Zhang, ~Xifeng~Guo, ~Sihang~Zhou and ~En~Zhu
\IEEEcompsocitemizethanks{
\IEEEcompsocthanksitem M. Luo, S. Wang, X. Liu, W. Tu, Y. Zhang, X. Guo, S. Zhou and E. Zhu are with School of Computer, National University of Defense Technology, Changsha, China, 410073 (e-mail: \{luomingjie13, \, wangsiwei13, \, xinwangliu, \, enzhu\}@nudt.edu.cn).}}

\maketitle

\begin{abstract}
Clustering is a fundamental task in the computer vision and machine learning community. Although various methods have been proposed, the performance of existing approaches drops dramatically when handling incomplete high-dimensional data (which is common in real world applications).
To solve the problem, we propose a novel deep incomplete clustering method, named \underline{D}eep \underline{D}istribution-preserving \underline{I}ncomplete \underline{C}lustering with \underline{O}ptimal \underline{T}ransport (DDIC-OT). To avoid insufficient sample utilization in existing methods limited by few fully-observed samples, we propose to measure distribution distance with the optimal transport for reconstruction evaluation instead of traditional pixel-wise loss function. Moreover, the clustering loss of the latent feature is introduced to regularize the embedding with more discrimination capability.
As a consequence, the network becomes more robust against missing features and the unified framework which combines clustering and sample imputation enables the two procedures to negotiate to better serve for each other. 
Extensive experiments demonstrate that the proposed network achieves superior and stable clustering performance improvement against existing state-of-the-art incomplete clustering methods over different missing ratios.

\end{abstract}

\section{Introduction}
Clustering is one of the fundamental and important unsupervised learning tasks in data science, image analysis and machine learning community \cite{liu2018late,liu2020efficient,peng2018structured,wu2019deep,kang2020partition,zhang2015low,liu2019multiple,yang2015sparse,zhang2020deep}. A wide variety of data clustering methods have been proposed to organise similar items into same groups and achieve promising performance, e.g., $k$-means clustering, Gaussian Mixture Model (GMM), spectral clustering and deep clustering recently. However, existing clustering approaches all hold one premise that the data themselves are complete while data with missing features are quite common in reality. Data incompleteness occurs due to many factors, e.g. sensor failure, unfinished collection and data storage corruption. For example, face images are covered with masks leading to missing features during COVID-19 \cite{greenhalgh2020face}. When facing with various types of missing features, incomplete data clustering has draw increasing attention in recent years \cite{tian2007face,li2014partial,wang2018partial,yang2019adaptive,liu2018late,liu2020efficient}(see Figure. \ref{Fig.problemsetting}).

\begin{figure}[t] 
\centering 
\includegraphics[width = 0.5\textwidth]{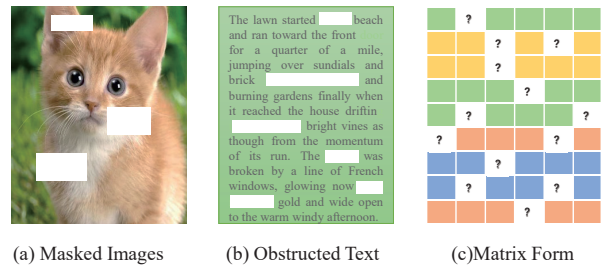} 
\vspace{-30pt}
\caption{(a) Image with missing features due to obstruction or shadow \cite{krizhevsky2012imagenet}; (b) Text data occupied with imperfect storage. (c) Data matrix with missing values. }
\label{Fig.problemsetting} 
\end{figure}

Existing incomplete clustering can be roughly categorized into two mechanisms, \textbf{heuristic-based} and \textbf{learning-based} respectively. Both of them firstly impute the missing features and then the full data matrix can be applied with traditional clustering algorithms. The heuristic imputation methods often rely on statistic property, e.g., zero-filling (ZF) and mean-filling (MF) after normalizing. Median values are also popular for imputation in genetic study. Particularly, the KNN-filling method considers the local reliable partners which fills the missing entries with the mean value of  $k$-closest neighbors. When facing with complex high-dimensional data, heuristic-based methods perform poorly since the simple imputations cannot obtain enough information to precisely recover data.

Recently, learning-based imputation methods receive enormous attention and become to be the mainstream. Existing work can be categorized into shallow and deep learning framework. The shallow representatives normally assume that the data are low-rank and therefore apply iterative methods to recover missing values \cite{nie2012low,jain2013low,wen2012solving,lu2014depth,fan2019factor}. Moreover, the Expectation-Maximum (EM) algorithm iteratively estimates the maximum likelihood and then inferences the missing variables until convergent. With the improvements of deep learning architectures, various deep networks have been proposed to handle incompleteness. A desirable attribute for deep approaches is that they should accurately inference the joint and marginal distributions of the data.  Therefore, variants of generation-style networks are introduced including Generative Adversarial Networks (GAN) and Variational Auto-Encoder (VAE). In \cite{yoon2018gain}, a generator utilizes the observed features to generate 'complete' data and the discriminator attempts to determine which components are actually observed or imputed. With the adversary training strategy, the generated missing features could approximate the real data distribution. Followed this line, enormous GAN and VAE-based approaches are put forward to minimizing the distances between real values and imputed matrices. 

Although these aforementioned methods offer solutions for incomplete data clustering, several drawbacks in existing mechanism cannot be neglected:
 i) Existing incomplete clustering methods follow a two-step manner, where the imputation stage and the clustering stage are separated from each other. In other words, the imputed features are not designed for clustering task, which may heavily degrade the clustering performance in return.
 ii) When facing with high-dimensional data (e.g., images, text), both of the shallow and deep methods perform poorly due to the insufficient observed information with inaccurate imputation. These results in sharp degradation in clustering task performance. 

In this paper, we propose a novel deep incomplete clustering method, which we refer as Deep Distribution-preserving Incomplete Clustering with Optimal Transport (DDIC-OT), that generalizes the well-known Deep Embedding Clustering network (DEC) to handle missing features. Different from existing pixel-by-pixel reconstruction in traditional autoencoder, we propose to minimize the Wasserstain distance between observed data and the reconstructed data with optimal transport. Moreover an addition clustering layer is added into the embedded representation level with KL-divergence for measuring clustering loss. By optimizing the novel network, the distribution of original data can be well-preserved and in return the missing features can be more accurately imputed by guidance of latent clustering structures. Thus, the proposed DDIC-OT simultaneously utilizes the imputation  and the embedded clustering procedures so that they can be jointly negotiated with each other and reach consensus best serving for clustering task. Finally, the proposed DDIC-OT is showcased in extensive experiments on a wide variety of benchmarks with different missing ratios, to evaluate its effectiveness. As demonstrated, the proposed network enjoys superior clustering performance in comparison with existing state-of-the-art imputation methods by large margins.

\begin{figure}[tb] 
\centering 
\includegraphics[width = 0.49\textwidth]{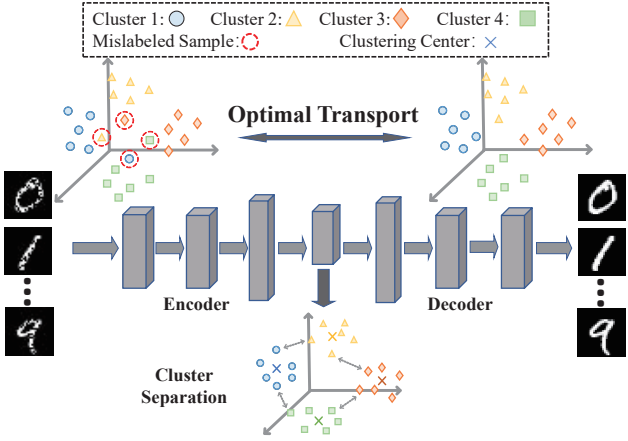} 
\caption{The framework of the proposed DDIC-OT. Instead of pixel-to-pixel reconstruction, we impute features by minimizing the distribution distance with optimal transport. Moreover, the latent embedding representations are regularized with clustering loss to ensure intra-cluster discrimination. The joint loss functions seamlessly negotiate incomplete imputation and clustering tasks as a unified framework to improve the performance of each other.}
\label{Fig.framework} 
\end{figure}

The contributions of our DDIC-OT are summarized as follows,
\begin{enumerate}
    \item We mathematically analyze the failures of existing incomplete clustering methods in theory when facing with high-dimensional data. To avoid insufficient training brought by the spareness of full-observed data, a novel end-to-end deep clustering network is proposed to  minimize the wasserstain distances between original and reconstructed distribution. 
    \item  We regularize the latent distribution with more discriminate separation to further enhance task performance. By the guidance of the unified loss function, the network decodes the informative latent representations contributing to better recovery and clustering. To the best of our knowledge, this could be the first work of end-to-end deep incomplete clustering network.
    \item Comprehensive experiments are conducted on six high-dimensional benchmarks datasets with various incomplete ratios. As the experimental results show, the proposed network significantly outperforms state-of-the-art incomplete clustering methods by large margins.
\end{enumerate}

\section{Notion and Related Work}
\subsection{Notion}
 Throughout this paper, we use boldface uppercase letters and lowercase letters to denote matrices and vectors respectively. The $(i,j)$-th elements of a matrix $\mathbf{U}$ is referred as $\mathbf{U}_{ij}$.

Let $\mathbf{X} \in \mathbb{R}^{n \times d}$ be the data matrix with $n$ samples $d$ dimension. Then we define a missing index matrix $\mathbf{M} \in \{0,1\}^{n \times d}$ as follows,
\begin{equation}
  \mathbf{M}_{ij}=\left\{\begin{array}{ll}
  1 & \text { if the entry } (i,j) \text { of } \mathbf{X} \text { can be observed, } \\
  0 & \text { otherwise. }
  \end{array}\right.
\end{equation}

With the defined mask matrix $\mathbf{M}$, the incomplete data matrix we observed can be defined as 
\begin{equation}
    \mathbf{X}^{(obs)} = \mathbf{X} \circ \mathbf{M} + NaN \circ (\mathbf{1}_{n \times d} - \mathbf{M}),
\end{equation}
where the $\circ$ denotes the Hadamard (elementwise) product and $NaN$ means \textbf{N}ot \textbf{a} \textbf{N}umber throughout the paper. 

\subsection{Related Work}

\noindent\textbf{Statistical imputation.} \quad Basic statistical methods try to utilize information from the missing data by the means of numerical property. Most of them use statistical attributes to estimate the missing feature values, rather than directly discard incomplete feature information. Incomplete entries are filled with constants to obtain complete data so that they can be directly applied to machine learning tasks, e.g., zero, mean and median. Additionally, KNN imputation method has been considered as an alternative estimating the missing features with the mean of $k$ nearest reliable neighbors \cite{crookston2008yaimpute}.

The Bayesian framework is different from the previous method in that it considers the joint and conditional distribution for dealing with incomplete features. These frames are generally expressed in terms of a maximum-likelihood method, which estimate missing values with the most probable numbers. The most popular method of Bayesian framework is the Expectation Maximization (EM) algorithm \cite{dempster1977maximum,ghahramani1994supervised}.

\noindent\textbf{Deep incomplete clustering.} \ Although deep clustering mechanism has received much attention in recent years, none of existing methods has considered to cluster with incomplete features in an end-to-end manner yet. The typical methods follow the two-step strategy:imputation and clustering separately. They propose to fill missing values through neural networks and then apply clustering algorithms on the estimated dataset. GAIN \cite{yoon2018gain} firstly proposes to impute incomplete features with GAN. Different from the traditional GAN networks, the goal of the discriminator in GAIN is to accurately distinguish whether the data are imputed or real, so as to force the samples generated by the generator to be close to the real data distribution. Unfortunately, the same problems as common GANs, these models are generally difficult to train since the optimization processes are hardly stable. Apart from GAN-like networks, VAEAC \cite{ivanov2019variational} proposes a neural probabilistic model based on variational autoencoder, which can estimate the observed features using stochastic gradient variational inference \cite{kingma2014stochastic}. However these VAE-based method may lead to poor results when the posterior approximation of variational inference is far from the actual posterior approximation. In addition, based on fitting the conditional distribution of the missing data, a Markov chain Monte Carlo (MCMC) scheme has been developed in \cite{Richardson_2020_CVPR}.  

\noindent \textbf{Optimal transport and sinkhorn divergence.} Let $\alpha=\sum_{i=1}^{n} a_{i} \delta_{\mathbf{X}_{i}}$, $\beta =\sum_{i=1}^{n} b_{i} \delta_{\mathbf{Y}_{i}}$ be two discrete distributions formed by empirical given data samples $\mathbf{X,Y}$, and their supports $\mathbf{X} \in \mathbb{R}^{n \times d}, \mathbf{Y} \in \mathbb{R}^{n^{\prime} \times d}$ and frequency vectors $\mathbf{a},\mathbf{b}$. It can be easily obtained that $\mathbf{a}^{\top} \mathbf{1} =1, \mathbf{a} \geq 0, \mathbf{b}^{\top} \mathbf{1} =1, \mathbf{b} \geq 0$.The $q$-th Wasserstein distance corresponds to these two distributions $\alpha$ and $\beta$ is denoted as follows,

\begin{equation}
\label{wasseatain-1}
    W_{q} (\alpha, \beta) \stackrel{\text { def }}{=} \min _{\mathbf{P} \in U(\mathbf{a}, \mathbf{b})}\langle\mathbf{F}, \mathbf{C}\rangle,
\end{equation}
where $U(\mathbf{a}, \mathbf{b}) \stackrel{\text { def }}{=}\left\{\mathbf{F} \in \mathbb{R}^{n \times n^{\prime}}_{+}: \mathbf{F} \mathbf{1} =\mathbf{a}, \mathbf{F}^{\top} \mathbf{1}_{n}=\mathbf{b}\right\}$ and $\mathbf{C}=\left(\left\|x_{i}-y_{j}\right\|^{q}\right)_{i j} \in \mathbb{R}^{n \times n^{\prime}}$ denotes as the cost matrix of pairwise squared distances between the support sets. In our paper, we set $q=2$. The Wasserstein distance denoted in Eq. (\ref{wasseatain-1}) is often jointly introduced with an entropy regularization, 
\begin{equation}
\label{wasseatain-2}
    W_{q}^{\epsilon} (\alpha, \beta) \stackrel{\text { def }}{=} \min _{\mathbf{F} \in U(\mathbf{a}, \mathbf{b})}\langle\mathbf{F}, \mathbf{C}\rangle - \epsilon h(\mathbf{F}),
\end{equation}
where  $h(\mathbf{F}) \stackrel{\text { def }}{=} -\sum_{i j} f_{i j} \log f_{i j}$ denotes the entropy regularization. Eq. (\ref{wasseatain-2}) can be efficiently optimized using Sinkhorn algorithm \cite{peyre2019computational}. Based on Eq. (\ref{wasseatain-2}), a symmetric divergence can be represented as
\begin{equation}
\label{sinkhorn}
S_{\epsilon}(\alpha, \beta) \stackrel{\text { def }}{=} \mathrm{OT}_{\epsilon}(\alpha, \beta)-\frac{1}{2}\left(\mathrm{OT}_{\epsilon}(\alpha, \alpha)+\mathrm{OT}_{\epsilon}(\beta, \beta)\right).
\end{equation} 

The Sinkhorn divergence in Eq. (\ref{sinkhorn}) offers a tractable alternative for Wasserstein distance calculations, and easily be accelerated by GPU. In our paper, we use the sinkhorn divergence to measure the OT distance of two distributions.


\section{DDIC-OT}
\subsection{Motivation}
\noindent \textbf{Problem analysis} Although the aforementioned methods have been proposed to solve incomplete data clustering to some extent, most of them are evaluated with very small-dimensional data and make them unpractical in real scenarios. When facing with high-dimensional data (e.g., images, text), both of the existing shadow and deep methods perform poorly due to the insufficient observed information with inaccurate imputation. We theoretically analyze this phenomenon with the following Theorem \ref{hard}.

\begin{theorem}
\label{hard}
Suppose the data are \textit{i.i.d (independently and identically distributed)}, a fully-observed high-dimensional data sample exists with low probability when facing incompleteness. 
\end{theorem}

\begin{proof}
Suppose the missing ratio is $p (0 \leq p \leq 1)$. Given a matrix $\mathbf{X} \in \mathbb{R}^{n \times d}$. For each sample $x_i$, we can obtain the following equation.
\begin{equation}
\begin{split}
 P\left(\mathbf{X}_{i1}, \ldots, \mathbf{X}_{id} \right)=P\left(\mathbf{X}_{i1}\right)  \cdots P\left(\mathbf{X}_{id}\right) = \left(1 - p\right)^d,
\end{split}
\end{equation}
where $P\left(\mathbf{X}_{i1}\right)$ denotes the probability $\mathbf{X}_{i1}$ can be observed.

Taking $p=0.1,d=300$ as an example, with $10^{-14}$ probability the sample $x_i$ can be fully-observed. With the increasing dimension $d$, the probability becomes smaller and approximates 0. This completes the proof. 
\end{proof}

The Theorem \ref{hard} illustrates that very few samples are fully complete when the dimensions are relatively high. Therefore, the traditional statistical and deep generative methods fail to impute proper values lacking of sufficient information, e.g., knn-filling and GAN-style solutions. In \cite{muzellec2020missing}, the Wasserstain distance is firstly applied to impute missing features where the assumption is to minimize the discrepancy of missing data distribution and the complete data distribution. In the low-dimensional incomplete setting (less than 50), the experiments results are promising and proven to show more stable along with change of the incomplete ratios. However, as confronting with much higher dimension data types (e.g., images, videos and text), very few fully-complete data can be obtained making the empirical estimation of target distribution (complete data distribution) difficult and inaccurate. Therefore these methods perform poorly in downstream clustering tasks (see results in Table. \ref{aggregated table}).

Different from existing assumptions, we propose to jointly solve the two processes in a unified framework: reconstruction and clustering. The natural way of reconstruction is to apply autoencoder models. The observed values can be regraded as 'supervised' signals for the reconstruction. However, with few informative information, it is not reasonable to only reconstruct the missing counterparts since the reconstruction may destroy the geometry distribution features for the data and the clustering performance is heavily affected. In this paper, we decide to recover the latent distribution instead of pixel-level approximation. Specially, we adopt the latent variable models defined by an encoder-decoder manner, where we firstly encode original data $x_i$ into the latent code $z_i$ in the  latent space $\mathcal{Z}$ and then $z_i$ is decoded to the reconstructed image $\hat{x}_i$. This process can be expressed as,

\begin{equation}
p_{\hat{X}}(\hat{x}):=\int_{\mathcal{Z}} p_{p_{\hat{X}}(\hat{x}|z)}(\hat{x}|z) p_{z}(z) d z, \quad \forall x \in \mathcal{X}
\end{equation}
where $p_{x}(z|x),p_{\hat{X}}(\hat{x}|z)$ are parameterized with the encoder $f_e$ and decoder $f_d$ network. Then the distribution-preserving loss can be measured with Eq. (\ref{sinkhorn}) respecting to $p_{X}$and $p_{\hat{X}}$,
\begin{equation}
\label{loss_s}
    L_{s}(\mathbf{X},\mathbf{\hat{X}}) = S_{\epsilon}(\mathbf{X},f_d(f_e(\mathbf{X}))).
\end{equation}

\subsection{Overall Network Architecture}
In this section, we leverage the one-stage deep incomplete clustering introduced in the previous section as a basis to demonstrate the process of the proposed learning algorithm, the overall flowchart is illustrated in Figure \ref{Fig.framework}. The proposed clustering model consists of three parts, an encoder, a decoder, and a soft clustering layer, specifically, the method relies on a linear combination based on two objective functions, representing the optimal transport distance and clustering loss respectively. The joint optimization process can be described as follows:
\begin{equation}
\label{loss}
L=L_{s}+\gamma L_{c},
\end{equation}
where $L_{s}$ is the sinkhorn divergence shown in Eq. (\ref{sinkhorn}) and $L_{c}$ is the clustering loss. $\gamma$ is a hyper-parameter, which is used to balance the two costs. 
Consider a dataset $\mathbf{X}$ with $n$ samples, and each $x_{i}$ $\in$ $\mathbb{R}^d$ where $d$ is the dimension. The number of clusters $k$ is known, for each input data $x_{i}$ we denote the nonlinear mapping $f_e: x_{i} \rightarrow z_{i}$ and $f_d: z_{i} \rightarrow \hat{x_{i}}$ where $z_{i}$ is the low dimensional feature space, $\hat{x_{i}}$ is the complete data learned through the network.

The clustering loss is defined as KL divergence between distributions $P$ and $Q$ proposed in\cite{xie2016unsupervised}, where $P$ is the soft assignment of the distribution $z$:
\begin{equation}
\label{P}
p_{i j}=\frac{\left(1+\left\|z_{i}-\mu_{j}\right\|^{2}\right)^{-1}}{\sum_{j}\left(1+\left\|z_{i}-\mu_{j}\right\|^{2}\right)^{-1}},
\end{equation}
and then the cluster assignment can be obtained $s_{i} =  \mathop{\arg\max}_{j} p_{i j}$. Then $Q$ is the target distribution derived from $P$,
\begin{equation}
\label{Q}
q_{i j}=\frac{p_{i j}^{2} / \sum_{i} p_{i j}}{\sum_{j}\left(p_{i j}^{2} / \sum_{i} p_{i j}\right)}.
\end{equation}
Therefore, the clustering loss is defined as 
\begin{equation}
\label{loss_c}
L_{c}=K L(Q \| P)=\sum_{i} \sum_{j} q_{i j} \log \frac{q_{i j}}{p_{i j}}.
\end{equation}

We summarize the merits of our proposed framework with the following factors: i) more naturally handle with incomplete clustering in high-dimensional space. The $L_s$ loss accomplishes the reconstruction samples with preserving geometry characteristics. ii) more flexible that does not require the prior distribution of $\mathbf{X}$ or $\mathbf{Z}$. Instead of explicit distribution formulation, our encoder and decoder network implicitly estimate the latent distribution with more flexibility; iii) regularizing the latent distribution $q$ with more discriminate separation to further enhance task performance. As the empirical experimental results show, the guidance of the joint loss function updates the network leading to the improvement of clustering performance. 

\subsection{Model Training}
The training phase of the model consists of two phases: the pre-training phase where the network only contains reconstruction loss and the fine-tune phase where both optimal transport distance and clustering loss are optimized. Our encoder consists of a fully-connected multi-layer perceptron with dimensions $d$-500-500-1000/2000-10 and the decoder is a mirrored version of the encoder. The details of the network architecture are provided in the supplementary materials. 
The optimizer Adam with init learning rate $\eta$ = 0.001 is applied for all datasets, and the batch size is set to 256. After pre-training, we provide two options of the parameter $\lambda$ 150 and 100. Furthermore, for sinkhorn divergence, we set entropy regularization parameters $\epsilon$ as 0.01. In addition, we will stop training if the percentage of label distribution change between two consecutive updates 
for target distribution is less than the threshold $\delta$ or the number of iterations meets $MaxIter$. Then the convergence threshold $\delta$ and $MaxIter$ are set to 0.1 and 200 separately. The full training procedure is summarized in Algorithm \ref{Algorithm-proposed}.

\alglanguage{pseudocode}
\begin{algorithm}[t]
\caption{DDIC-OT}
\label{Algorithm-proposed}
\hspace*{0.02in}{\bf Input:}  
Missing data $\mathbf{X}_m$; Cluster number $k$; Hyper-parameter $\lambda$; Batchsize $N$; Maximum iterations $MaxIter$; Stopping threshold $\delta$; Learning rate $\eta$. \\
\hspace*{0.02in}{\bf Output:} Clustering Assignment $S$.
\begin{algorithmic}[1]
\State Initialize $\mathbf{X}_m$ with mean filling.
\State Initialize clustering centroids $u$.
\For{$iter = 0$ to $Maxiter$}
\For{$i = 0$ to $\lfloor n/N \rfloor$}
    \State Sample a minibatch $\left\{x_{i}\right\}_{i=1}^m$ from $\mathbf{X}$.
    \State Compute related variables by $z_{i}$ = $f_{e}(x_{i})$, $\hat{x_{i}}$ = \State $f_{d}(z_{i})$.
    \State Compute $P_{i}$, $Q_{i}$ using Eq. (\ref{P}) and Eq. (\ref{Q})
    \State Compute clustering assignment for $\left\{x_{i}\right\}_{i=1}^m$.
    \State Compute overall loss $L$ by Eq. (\ref{loss}).
    \State Back-propagation and update model weights. 

\EndFor
    \State Compute $Z$ = $f_{e}(\mathbf{X})$.
    \State Compute $P$ and $Q$. 
    \State Compute clustering assignment $S$.
    \If{$sum(S_{iter+1}\neq S_{iter})/n$ \textless $\delta$}
    \State    Stop training.
    \EndIf
\EndFor
\end{algorithmic}
\end{algorithm}

\begin{table*}[htb]
\caption{The aggregated ACC, NMI and Purity comparison (mean$\pm$std) of different algorithms on  benchmark datasets. '-' means out of the GPU memory. The detailed results are omitted due to space limit and provided in supplementary materials.}
\label{aggregated table}
\centering
\resizebox{\textwidth}{!}{
\begin{tabular}{|c||c|c|c|c|c|c|c|c|c|c|}
\toprule
Method & \multicolumn{5}{|c|}{Shallow} & \multicolumn{5}{|c|}{Deep} \\
\hline
Dataset & MF    & ZF    & LRC   & MNC   & FSGR  & GAIN  & VAEAC & MIWAE & MDIOT & Ours \\\hline
\multicolumn{11}{|c|}{ACC$(\%)$}            \\\hline
Mnist   & $54.66\pm3.13$ & $52.48\pm3.33$ & $51.82\pm3.46$ & $53.28\pm3.08$ & $53.22\pm2.97$ & $52.05\pm3.06$ & $53.73\pm3.63$ & -  & $54.31\pm3.46$ & $\mathbf{84.31\pm0.0}$    \\\hline
Usps    & $61.63\pm3.92$ & $60.72\pm3.24$ & $61.78\pm3.63$ & $61.80\pm3.61$ & $60.22\pm3.38$ & $60.20\pm3.11$ & $61.90\pm3.51$ & $48.75\pm4.73$ & $63.63\pm4.07$ & $\mathbf{74.48\pm0.0}$    \\\hline
Fmnist  & $51.14\pm3.87$ & $49.56\pm3.84$ & $51.70\pm3.75$ & $51.56\pm3.44$ & $52.12\pm4.16$ & $50.75\pm4.39$ & $52.04\pm4.05$ & -  & $50.91\pm3.54$ & $\mathbf{58.88\pm0.0}$    \\\hline
Reuters & $56.26\pm11.67$ & $50.96\pm9.15$ & $53.79\pm5.48$ & $54.17\pm5.27$ & $53.72\pm5.47$ & $51.25\pm9.49$ & $60.22\pm8.62$ & $54.51\pm10.00$ & $52.18\pm8.21$ & $\mathbf{76.06\pm0.0}$   \\\hline
COIL20  & $54.33\pm4.95$ & $42.77\pm5.86$ & $58.73\pm4.56$ & $59.10\pm4.47$ & $55.00\pm4.74$ & $55.10\pm4.42$ & $60.21\pm4.22$ & $57.87\pm5.13$ & $58.95\pm4.67$ & $\mathbf{66.40\pm0.0}$    \\\hline
Letter  & $35.77\pm1.22$ & $33.40\pm1.47$ & $36.95\pm1.36$ & $33.68\pm1.53$ & $37.34\pm1.05$ & $35.54\pm1.23$ & $36.11\pm1.32$ & $27.42\pm0.91$ & $35.40\pm1.39$ & $\mathbf{47.15\pm0.0}$    \\\hline
\multicolumn{11}{|c|}{NMI$(\%)$}          \\\hline
Mnist   & $47.82\pm1.42$ & $45.48\pm1.53$ & $46.07\pm1.59$ & $46.55\pm1.36$ & $46.57\pm1.27$ & $46.06\pm1.31$ & $47.62\pm1.19$ & -  & $49.62\pm1.07$ & $\mathbf{76.84\pm0.0}$    \\\hline
Usps    & $58.51\pm3.92$ & $55.89\pm3.24$ & $57.56\pm3.63$ & $58.02\pm3.61$ & $56.21\pm3.38$ & $57.54\pm1.51$ & $58.32\pm1.25$ & $44.60\pm4.38$ & $62.33\pm1.37$ & $\mathbf{74.48\pm0.0}$    \\\hline
Fmnist  & $49.83\pm3.87$ & $47.03\pm3.84$ & $49.98\pm3.75$ & $50.07\pm3.44$ & $50.19\pm4.16$ & $50.55\pm1.16$ & $49.93\pm1.08$ & -  & $49.49\pm1.03$ & $\mathbf{61.28\pm0.0}$    \\\hline
Reuters & $26.37\pm11.67$ & $21.41\pm9.15$ & $26.43\pm5.48$ & $27.37\pm5.27$ & $25.82\pm5.47$ & $21.26\pm9.94$ & $31.99\pm8.04$ & $23.71\pm9.67$ & $22.28\pm11.42$ & $\mathbf{44.92\pm0.0}$    \\\hline
COIL20  & $68.89\pm4.95$ & $55.75\pm5.86$ & $73.38\pm4.56$ & $74.05\pm4.47$ & $70.22\pm4.74$ & $69.14\pm2.67$ & $74.36\pm2.38$ & $71.72\pm2.64$ & $73.43\pm2.10$ & $\mathbf{77.72\pm0.0}$    \\\hline
Letter  & $37.58\pm1.22$ & $34.79\pm1.47$ & $39.18\pm1.36$ & $35.06\pm1.53$ & $39.43\pm1.05$ & $37.98\pm0.57$ & $38.56\pm0,57$ & $30.04\pm0.63$ & $37.45\pm0.72$ & $\mathbf{52.29\pm0.0}$    \\\hline
\multicolumn{11}{|c|}{Purity$(\%)$}          \\\hline
Mnist   & $58.37\pm1.62$ & $56.64\pm2.21$ & $57.35\pm1.79$ & $57.96\pm1.83$ & $58.09\pm1.81$ & $56.78\pm1.89$ & $57.91\pm1.57$ & -  & $59.28\pm1.47$ & $\mathbf{84.31\pm0.0}$   \\\hline
Usps    & $69.40\pm2.47$ & $67.74\pm2.74$ & $69.11\pm2.47$ & $69.55\pm2.00$ & $67.70\pm2.29$ & $67.72\pm2.34$ & $69.36\pm2.45$ & $55.40\pm5.35$ & $71.35\pm2.77$ & $\mathbf{80.19\pm0.0}$    \\\hline
Fmnist  & $56.09\pm2.20$ & $53.03\pm2.84$ & $56.71\pm2.22$ & $56.61\pm2.42$ & $57.15\pm1.52$ & $55.98\pm2.16$ & $56.62\pm1.87$ & -  & $56.27\pm1.88$ & $\mathbf{63.47\pm0.0}$    \\\hline
Reuters & $74.71\pm5.10$ & $74.21\pm4.97$ & $78.41\pm4.08$ & $\mathbf{78.97\pm3.59}$ & $77.79\pm3.91$ & $73.78\pm5.12$ & $78.18\pm3.94$ & $74.43\pm4.99$ & $73.67\pm4.76$ & $75.72\pm0.0$    \\\hline
COIL20  & $58.03\pm5.49$ & $45.63\pm4.97$ & $63.25\pm3.93$ & $63.96\pm4.09$ & $59.38\pm3.83$ & $58.77\pm4.19$ & $64.63\pm4.16$ & $61.30\pm4.04$ & $63.48\pm3.67$ & $\mathbf{70.94\pm0.0}$    \\\hline
Letter  & $37.90\pm1.01$ & $35.23\pm1.89$ & $39.30\pm1.02$ & $35.49\pm1.67$ & $39.76\pm0.99$ & $38.12\pm0.99$ & $38.68\pm0.88$ & $29.04\pm0.83$ & $37.66\pm1.14$ & $\mathbf{49.56\pm0.0}$  \\
\bottomrule
\end{tabular}}
\end{table*}

\section{Experiments}

\subsection{Experiments Setup}
\noindent\textbf{Datasets}
In this paper, we conduct extensive experiments on the six widely-used large-scale benchmark datasets. (1) \textbf{MNIST-full} and \textbf{Fashion-MNIST}\cite{lecun1998gradient}: 70000 images including the training and testing split are combined into a unified dataset. (2)\textbf{USPS}\cite{hull1994database}: This dataset contains a total of 9298 grayscale samples with 16 $\times$ 16 pixels. (3) \textbf{COIL-20}\footnote{https://www.cs.columbia.edu/CAVE/software/softlib/}:COIL-20 consists of 1440 images of 20 objects taken by cameras from varying angles. (4)\textbf{Reuters-10K}:  We used 4 root categories: corporate/industrial, government/social, markets and economics as labels and excluded all documents with multiple labels. We randomly sampled a subset of 10000 examples and computed tf-idf features on the 2000 most frequent words. We term this dataset as Reuters-10K.(5)\textbf{Letter} \footnote{https://www.nist.gov/itl/products-and-services/emnist-dataset}: The Letter dataset merges a balanced set of the 26 letters with 800 images each class.  

Followed by existing incomplete clustering task setting,  we set seven groups of incomplete ratios as $\{0.1,0.2,0.3,\cdots 0.6,0.7\}$ for each dataset in our experiments. Incomplete ration means the percentage of missing features in all samples.

\begin{table}[h]
\centering
\label{dataset table}
\caption{Benchmark dataset description.}
\begin{center}
\begin{tabular}{|c|c|c|c|c|}
\hline
Dataset & Samples & Dimensions & Classes   \\
\hline
\hline
Mnist   & 70000                         & 784                           & 10    \\
USPS    & 9298                          & 256                           & 10     \\
FMNIST  & 70000                         & 784                           & 10     \\
Reuters-10k & 10000                     & 2000                          & 4    \\
COIL-20 & 1440                          & 1024                          & 20   \\
Letter  & 20800                         & 784                           & 26     \\
\hline
\end{tabular}
\end{center}
\end{table}

\noindent\textbf{Evaluation Metrics} \quad
In our experiements, we used three standard clustering performance metrics for evaluation: (1) \textbf{Accuracy (ACC)} is computed by
assigning each cluster with the dominating class label and
taking the average correct classification rate as the final
score, (2) \textbf{Normalised Mutual Information (NMI)} quantifies the normalised mutual dependence between the predicted labels and the ground-truth, and (c) \textbf{Purity} measures the proportion of the number of samples correctly clustered to the total number of samples. All of these metrics scale from 0 to 1 and higher values indicate better performance. Specially, we report the mean values and standard derivations of 10 independent runs to avoid the randomness brought by the different initializations of $k$-means.

\subsection{Compared SOTA methods}
\noindent(1) $\textbf{Mean-Filling (MF)}$: The missing features are imputed with the mean of the observed values in the corresponding dimensions.(2) $\textbf{Mean-Filling (MF)}$: The missing features are imputed with zeros in the normalized data matrix.(3) $\textbf{Low-rank Completion(LRC)}$\cite{nie2012low}: The method attempts to recover data matrix with low-rank assumption. (4) $\textbf{Max Norm Completion (MNC)}$\cite{fan2019factor}: MNC adopts the max-norm to complete missing features.(5)\textbf{Factor Group-Sparse Regularization for Efficient Low-Rank Matrix Recovery(FSGR)}\footnote{https://github.com/udellgroup/Codes-of-FGSR-for-effecient-low-rank-matrix-recovery}\cite{fan2019factor}   The author proposes factor group-sparse regularizers to accomplish low-rank matrix completion task.(6)$\textbf{GAIN}$\footnote{https://github.com/jsyoon0823/GAIN}\cite{yoon2018gain}: \textbf{Missing data imputing using Generative Adversarial Nets}. It proposes a method that uses GAN to estimate and complete the work of filling missing values. (7)$\textbf{ VAEAC:}\footnote{https://github.com/tigvarts/vaeac}$\cite{ivanov2019variational}\textbf{Variational Autoencoder with Arbitrary Conditioning}.  It is a latent variable model trained using stochastic gradient variational Bayes.(8)$\textbf{MIVAE}$\footnote{https://github.com/pamattei/miwae}\cite{mattei2019miwae}: MIWAE is based on the importance-weighted autoencoder, and maximises a potentially tight lower bound of the log-likelihood of the observed data. (9)$\textbf{MDIOT}$\footnote{https://github.com/BorisMuzellec/MissingDataOT}\cite{muzellec2020missing}: \textbf{Missing Data Imputing using Optimal Transport.} This paper leverages OT to define a loss function for missing data distribution and complete data distribution. The hyper-parameters used in all our comparative experiments follow their corresponding papers.

For all the compared methods above, we have downloaded their public implementations with Matlab and Pytorch. All our experiments are conducted on desktop computer  with Intel i7-9700K CPU @ 3.60GHz$\times$12, 64 GB RAM and GeForce RTX 3090 25GB.

\begin{figure*}[tbp] 
\centering 
\includegraphics[height= 0.5\textheight,width = 1\textwidth]{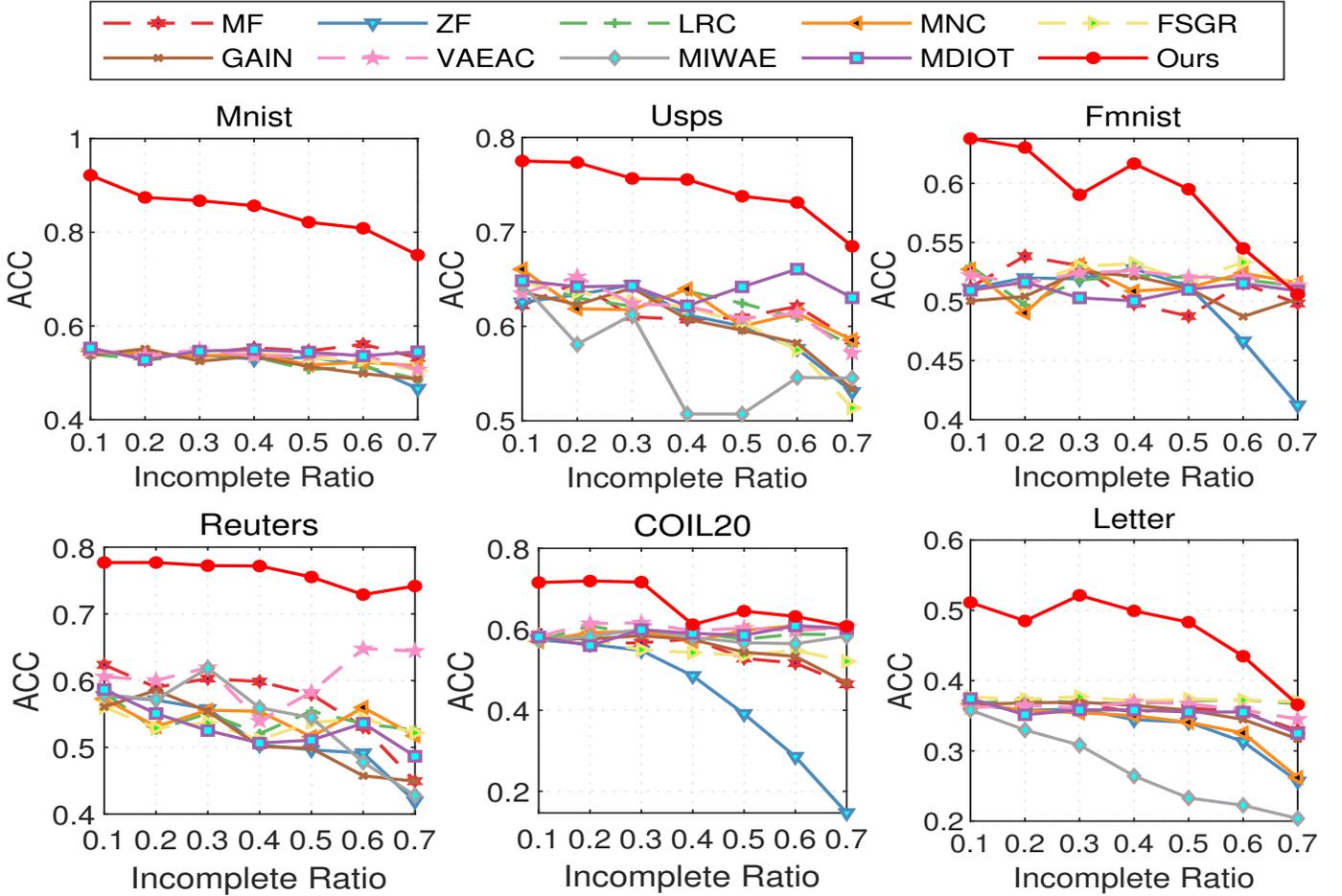} 
\caption{The clustering results of ACC metric on the benchmark datasets with different incomplete ratios. The results of Purity are provided in supplementary materials due to space limit. }
\label{Fig.variousaccratio} 
\end{figure*}

\subsection{Results Comparisons to Alternative Methods}
Table \ref{aggregated table} shows the aggregated clustering comparison of the above algorithms on the benchmark datasets. The best results are highlighted with boldface and '-' means the out of GPU memory failure. Based on the results, we have the following observations:
\begin{itemize}
\item Our proposed method outperforms all the SOTA imputation competitors in clustering performance by large margins. For example, our algorithm surpasses the second best by $\textbf{50.4\%,\,17\%,\,12\%,\,26\%,\,10\%}$and $\textbf{30\%}$, in terms of ACC on all benchmark datasets. In particular, the margins for the four datasets (Mnist, Usps, Reuters and Letter) are very impressive. These results clearly verify the effectiveness of the proposed network.
\item Comparing with the generative-style methods, the proposed DDIC-OT consistently further improves the clustering performance and achieves better results among the benchmark datasets. GAIN, VAEAC and MIWAE are the chosen representative methods. As can be seen, they concentrate on the generation or imputation task while ignoring the impacts of downstream clustering procedure. The joint optimization framework further contributes to improving performance.
\item MDIOT has been considered as a strong baseline for incomplete data imputation. It outperforms other competitors among most of the datasets. Our proposed algorithm surpasses MDIOT by $\textbf{55.2\%,\,17.1\%,\,15.7\%,\,45.8\%,\,12.6\%}$and $\textbf{33.2\%}$ in terms of ACC on all benchmark datasets. The phenomenon demonstrates the effectiveness of the our proposed architecture. Regardless of directly computing distribution distances in original space, the bottleneck of our embedding layer serves for clustering task and make distributions more discriminate. 
\end{itemize}

\begin{figure*}[btp] 
\centering 
\includegraphics[height= 0.55\textheight,width = 1\textwidth]{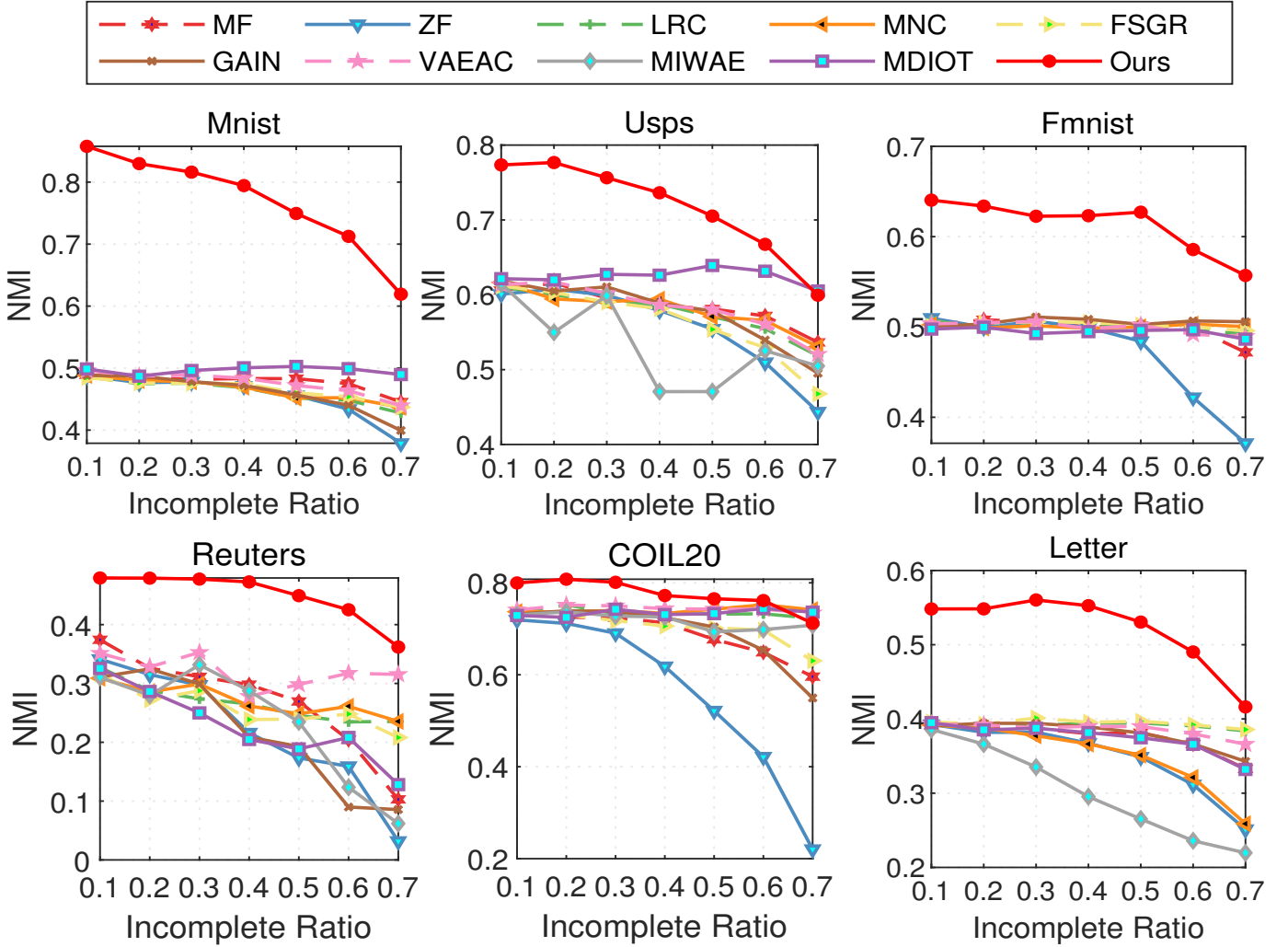} 
\caption{The clustering results of NMI metric on the benchmark datasets with different incomplete ratios. }
\label{Fig.variousnmiratio} 
\end{figure*}

\subsection{Qualitative Study}
In this section, we deeply analyze the clustering performance regarding to various ratios and the evolution of the learned representation. In order to show the comparison between different methods more clearly, we draw the ACC and NMI of compared methods under different missing rates as line graphs as shown in Figure. \ref{Fig.variousaccratio} and \ref{Fig.variousnmiratio}.

From the figure, we can obtain the following observations:\textbf{(1)} As can be seen, with the incomplete ratios increasing, all the methods suffers the degradation of clustering performance due to more unavailable information. Especially for the generative-based methods (VAEAC and MDIOT), their performance drops sharply due to inaccurate imputations.
\textbf{(2)} The results of our proposed method in terms of ACC are higher than all the competing algorithms for different incomplete ratios. Moreover, our method achieves stable performance against the increasing incomplete ratios. These results clearly demonstrates the effectiveness of DDIC-OT. 
\textbf{(3)}  We also show the relative NMI performance of the compared methods in Figure. \ref{Fig.variousnmiratio}. As can be seen, the clustering performance results are consistent with the ACC observations.
\subsection{Ablation Study}

\noindent\textbf{Loss Ablation Study}

We first investigate how the clustering loss and the distribution-preserving loss affect the clustering performance on Mnist/Usps/Reuters, and the results are shown in Table \ref{loss ablation}. In this experiment, we uniformly adopt datasets with 10$\%$ missing ratio. It seems that the $L_{c}$ has more contributions than $L_{s}$ on Mnist/Usps for clustering, and inversely on Reuters. We also conclude that the joint of two counterpart losses further contributes to better performance.

\begin{table}[]
\caption{Loss ablation study with 10$\%$ missing ratio. }
\label{loss ablation}
\vspace{-5pt}
\begin{center}
\resizebox{0.49\textwidth}{!}{
\begin{tabular}{|c|c|c|c|c|c|c|c|}
\hline
Dataset & \multicolumn{2}{c|}{Mnist} & \multicolumn{2}{c|}{Usps}  & \multicolumn{2}{c|}{Reuters}    \\
\hline
Loss& ACC & NMI & ACC & NMI & ACC & NMI \\
\hline
$L_{s}$      & 72.24 & 62.36  & 55.59 & 53.29 & 76.47 & 44.58 \\
$L_{c}$ & 86.84 & 77.7   & 74.81 & 73.88 & 72.42 & 43.09 \\
$L_{s}$ + $L_{c}$  & $\mathbf{92.16}$ & $\mathbf{85.77}$  & $\mathbf{77.5}$  & $\mathbf{77.35}$ & $\mathbf{77.7}$  & $\mathbf{47.93}$ \\
\hline
\end{tabular}}
\end{center}
\end{table}

\noindent\textbf{Sensitivity to initialization imputed values}

The initialization of imputed values has been demonstrated to be an essential part of incomplete clustering. We tested its sensitivity in our DDIC-OT, w.r.t. model performance on Mnist/Usps/Reuters. We evaluated two commonly-used initialization values: zero-filling (ZF) and mean-filling (MF). Table \ref{Different Initializations} shows that DDIC-OT can work stably without clear variation in the overall performance when using different initializations. This verifies that our method is insensitive to network initialization.

\begin{table}[]
\caption{Model sensitivity to different initializations of imputed values of three different missing ratios (10\%/30\%50\%) on three benchmarks. Metric:ACC.}
\label{Different Initializations}
\vspace{-5pt}
\begin{center}
\resizebox{0.49\textwidth}{!}{
\begin{tabular}{|c|c|c|c|c|c|c|c|}
\hline
Dataset & \multicolumn{2}{c|}{Mnist} & \multicolumn{2}{c|}{Usps}  & \multicolumn{2}{c|}{Reuters}    \\
\hline
Initializations& ZF & MF & ZF & MF & ZF & MF \\
\hline
10\% & 90.75 & 92.16 & 74.19 & 77.5  & 74.03 & 77.7  \\
30\% & 83.36 & 86.74 & 72.46 & 75.65 & 72.47 & 77.23 \\
50\% & 78.61 & 82.15 & 71.77 & 73.77 & 72.03 & 75.53 \\
\hline
\end{tabular}}
\end{center}
\end{table}

\section{Conclusion}

In this paper, we propose a novel incomplete clustering methods termed DDIC-OT, which jointly performs clustering and missing data imputation into a unified framework. Extensive experiments are conducted to demonstrate the effectiveness of optimal transport for clustering tasks. In the future, we will consider to construct more advanced network to further improve incomplete clustering performance.


{\small
\bibliographystyle{ieee_fullname}
\bibliography{egbib}
}

\newpage
\section{Appendix}
In this section, we mainly supplement our work from two aspects, namely the specific settings of the experiment and the display of more experimental results.

\subsection{Model Training}
We summarize in Table. \ref{experiment settings} the dataset specific optimal values of the hyperparameter $\gamma$ which trades off between the optimal transport and the clustering loss, as well as the full connection of two different node numbers selected due to the difference in the dimensional of the dataset, in particular, we only change the last layer of the encoder and the corresponding decoding layer. 

\subsection{More Experimental Results}
\subsubsection{Sensitivity to initialization imputed values}
The initialization of imputed values has been demonstrated to be an essential part of incomplete clustering. We further tested its sensitivity in our DDIC-OT, w.r.t. model performance on Fmnist/COIL20/Letter. Experimental results show that not all the datasets that perform with mean-filling are better than that with zero-filling, for example, when Fmnist with 10\% missing ratio, the effect of zero-filling is much better. Then the Table. \ref{Different Initializations} shows that when different initializations are used, DDIC-OT can run stably without significant changes in overall performance.

\begin{table}[H]
\caption{Model sensitivity to different initializations of imputed values of three different missing ratios (10\%/30\%50\%) on three benchmarks. Metric:ACC.}
\label{Different Initializations}
\vspace{-5pt}
\begin{center}
\resizebox{0.49\textwidth}{!}{
\begin{tabular}{|c|c|c|c|c|c|c|c|}
\hline
Dataset & \multicolumn{2}{c|}{Fmnist} & \multicolumn{2}{c|}{COIL20}  & \multicolumn{2}{c|}{letter}    \\
\hline
Initializations& ZF & MF & ZF & MF & ZF & MF \\
\hline
10\% & 62.92 & 63.79 & 70.49 & 71.58  & 47.44 & 51.12  \\
30\% & 62.59 & 59.03 & 70.56 & 71.67 & 47.37 & 52.14 \\
50\% & 60.21 & 59.5 & 65.35 & 64.51 & 47.75 & 48.32 \\
\hline
\end{tabular}}
\end{center}
\end{table}

\subsubsection{Comprehensive Experimental Results}
To provide more comprehensive experimental results about the performance of our algorithm under different missing ratios, we list the ACC, NMI, PUR of each dataset with a missing ratio ranging from 10\% to 70\% in Table. \ref{mnist},  \ref{usps}, \ref{fmnist}, \ref{Reuters}, \ref{coil20} and \ref{letter}. It is clearly to see that our algorithm has achieved a apparent lead among almost all the missing proportions of the six benchmarks. In order to show the comparison between different methods more clearly, we draw the Purity of compared methods under different missing rates as line graphs as shown in Figure \ref{Fig.variouspurratio}. We observe that our algorithm has advantages for all data sets except Reuters. Moreover, we can find that the performance of our algorithm on Reuters has obvious advantages on the two indicators of ACC and MNI. By analyzing the situation where ACC and MNI are much smaller than PUR, we can clearly see that our comparison algorithm classifies most samples into one cluster instead of achieving really fine clustering results.

\begin{figure*}[!tbp] 
\centering 
\includegraphics[height= 0.4\textheight,width = 1\textwidth]{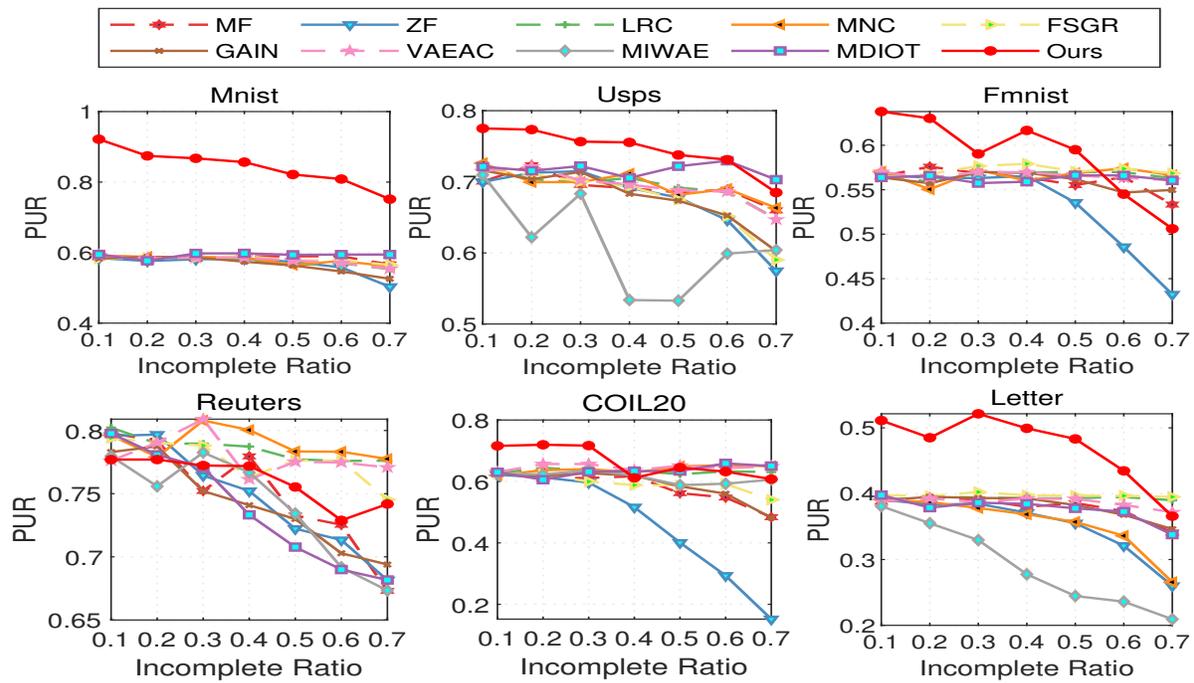} 
\caption{The clustering results of Purity metric on the benchmark datasets with different incomplete ratios.}
\label{Fig.variouspurratio} 
\end{figure*}

\subsection{Evolution and Convergence}

\begin{figure}[H] 
\centering 
\subfloat[Mnist]{\includegraphics[width = 0.16\textwidth]{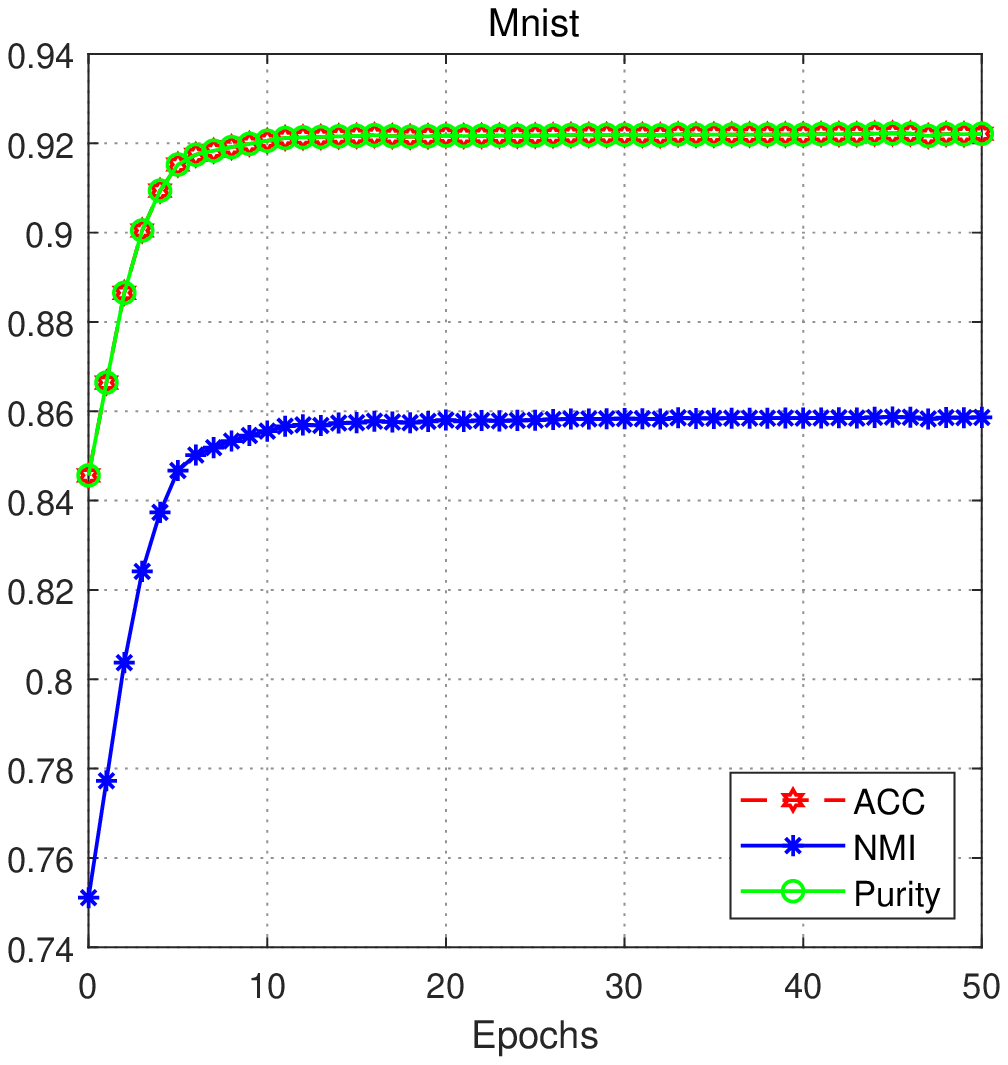}}
\subfloat[Fmnist]{\includegraphics[width = 0.16\textwidth]{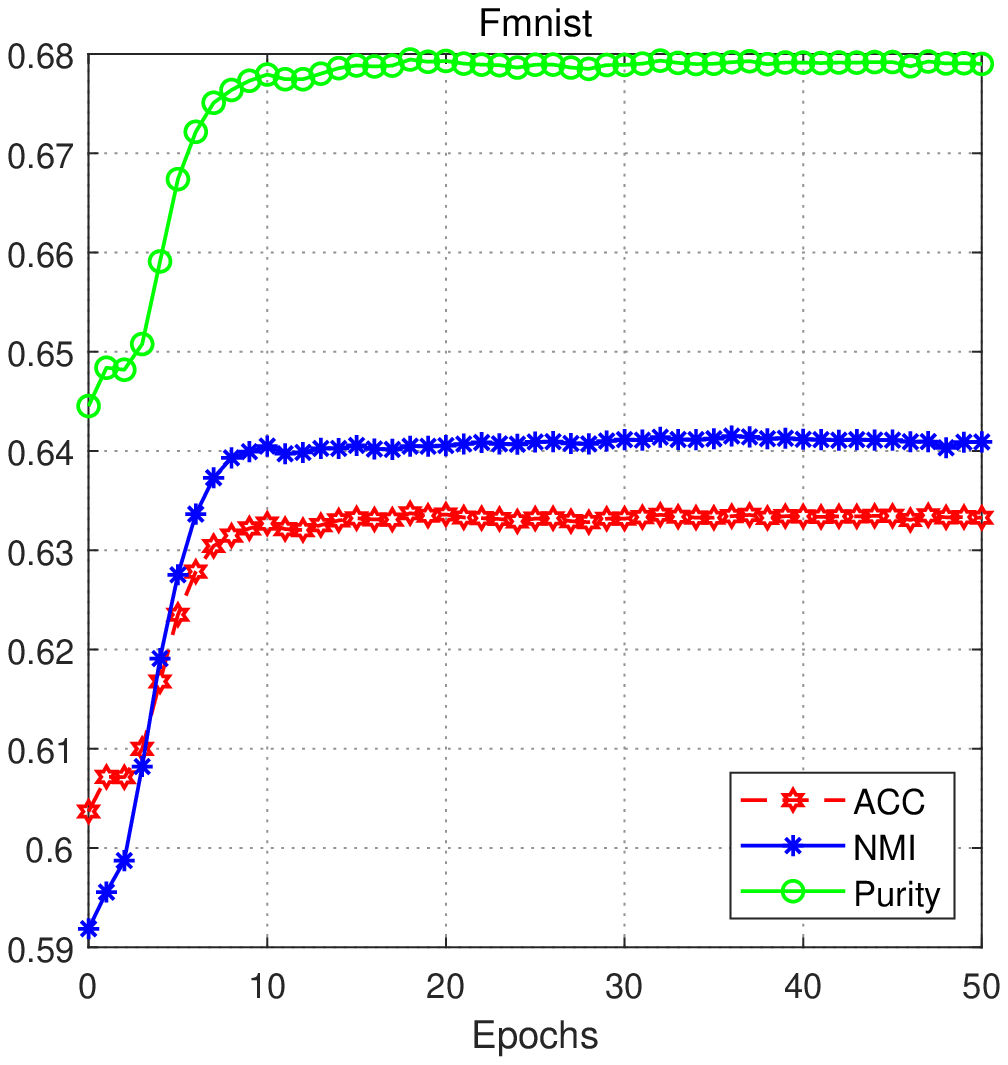}}
\subfloat[Usps]{\includegraphics[width = 0.16\textwidth]{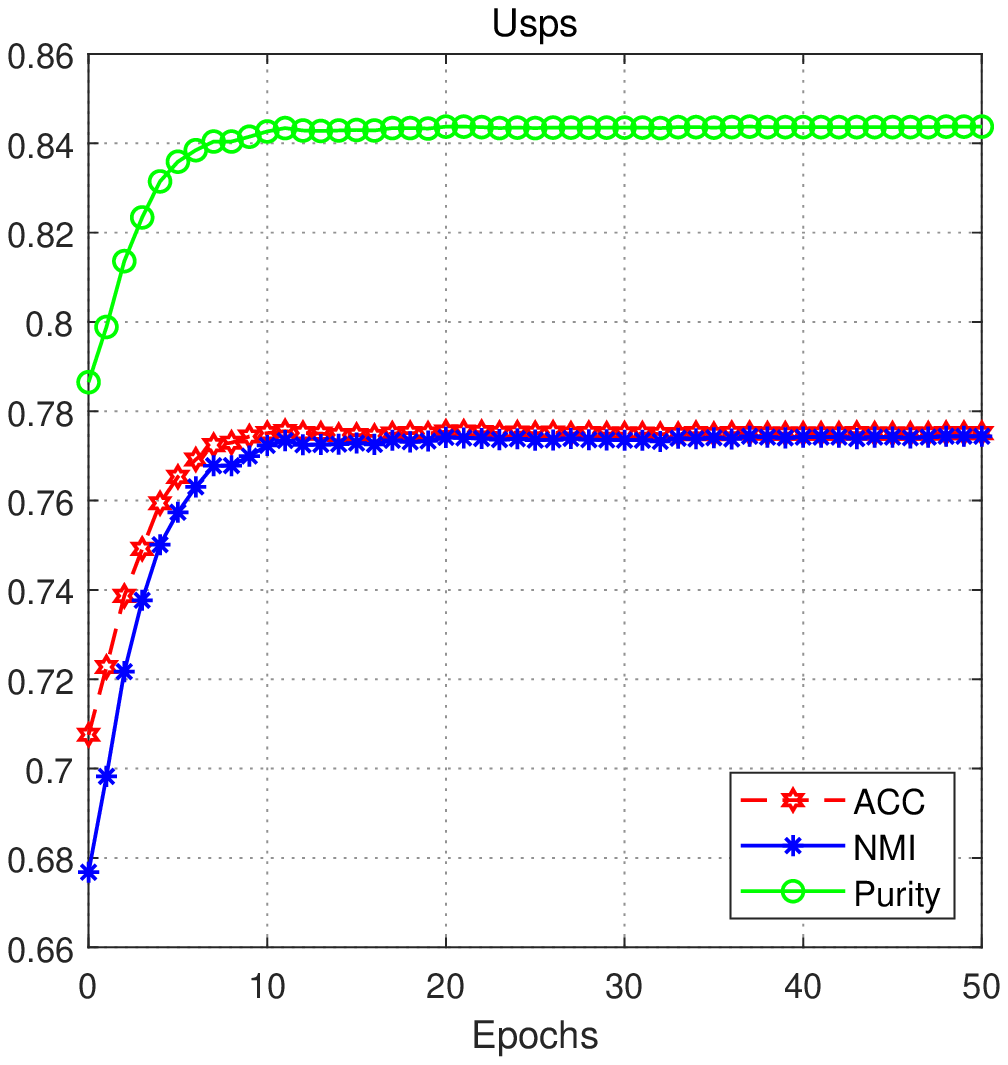}} \\
\caption{The clustering performance evolution with respecting to training epochs when missing ratio is $0.1$.}
\label{Fig.evolution} 
\end{figure}

A large number of experimental results above have proved the effectiveness of the proposed algorithm. In order to demonstrate more clearly about how our model converges to the target, the clustering performance evolution with respecting to training epochs is shown in Figure \ref{Fig.evolution}. Specifically, we show in the figure the changes of ACC, MNI, and PUR with the training epoch when the missing rate of the three datasets is 10\%. It is obviously that our algorithm achieves fast down convergence and satiafactory evolution.

\subsubsection{Qualitative Study}
Figure \ref{imputation results} illustrates the complete data and the imputation performance of the competing method with 30\% missing rate on the Mnist. Images along rows (a) and (b) contains the complete data images and initialized data images which represents mean-imputed data at 30\% missing rate, respectively. The compared imputation methods take the observed images (b) for input, and the complete images are shown for reference. Images along rows (c), (d) include the imputed results after using GAIN, MIDOT respectively. Compared with GAIN and MIDOT, it can be seen from row (e) that our method is visually impressive. More importantly, different from the pixel-level reconstruction used in GAIN and MIDOT, our work is better at preserving the structure of the intended image which is benefit for clustering task. Specifically, the embedding layer has learned the structural information and removed the noise from our network, and this is why our method can obtain leading clustering effects in the face of missing data. The experimental results in Table. \ref{mnist} also verify our statement.

\begin{figure*}[!tbp] 
\centering 
\subfloat{\includegraphics[width = 0.09\textwidth]{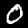}}
\subfloat{\includegraphics[width = 0.09\textwidth]{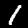}}
\subfloat{\includegraphics[width = 0.09\textwidth]{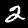}}
\subfloat{\includegraphics[width = 0.09\textwidth]{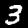}}
\subfloat{\includegraphics[width = 0.09\textwidth]{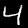}}
\subfloat{\includegraphics[width = 0.09\textwidth]{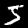}}
\subfloat{\includegraphics[width = 0.09\textwidth]{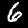}}
\subfloat{\includegraphics[width = 0.09\textwidth]{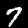}}
\subfloat{\includegraphics[width = 0.09\textwidth]{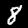}}
\subfloat{\includegraphics[width = 0.09\textwidth]{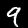}}\\
(a) Complete data images\\
\subfloat{\includegraphics[width = 0.09\textwidth]{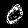}}
\subfloat{\includegraphics[width = 0.09\textwidth]{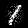}}
\subfloat{\includegraphics[width = 0.09\textwidth]{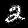}}
\subfloat{\includegraphics[width = 0.09\textwidth]{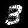}}
\subfloat{\includegraphics[width = 0.09\textwidth]{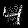}}
\subfloat{\includegraphics[width = 0.09\textwidth]{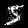}}
\subfloat{\includegraphics[width = 0.09\textwidth]{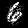}}
\subfloat{\includegraphics[width = 0.09\textwidth]{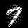}}
\subfloat{\includegraphics[width = 0.09\textwidth]{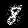}}
\subfloat{\includegraphics[width = 0.09\textwidth]{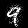}}\\
(b) Initialized data images\\
\subfloat{\includegraphics[width = 0.09\textwidth]{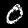}}
\subfloat{\includegraphics[width = 0.09\textwidth]{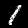}}
\subfloat{\includegraphics[width = 0.09\textwidth]{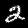}}
\subfloat{\includegraphics[width = 0.09\textwidth]{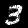}}
\subfloat{\includegraphics[width = 0.09\textwidth]{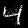}}
\subfloat{\includegraphics[width = 0.09\textwidth]{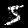}}
\subfloat{\includegraphics[width = 0.09\textwidth]{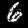}}
\subfloat{\includegraphics[width = 0.09\textwidth]{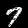}}
\subfloat{\includegraphics[width = 0.09\textwidth]{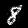}}
\subfloat{\includegraphics[width = 0.09\textwidth]{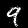}}\\
(c) GAIN\\
\subfloat{\includegraphics[width = 0.09\textwidth]{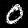}}
\subfloat{\includegraphics[width = 0.09\textwidth]{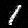}}
\subfloat{\includegraphics[width = 0.09\textwidth]{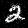}}
\subfloat{\includegraphics[width = 0.09\textwidth]{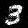}}
\subfloat{\includegraphics[width = 0.09\textwidth]{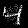}}
\subfloat{\includegraphics[width = 0.09\textwidth]{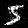}}
\subfloat{\includegraphics[width = 0.09\textwidth]{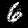}}
\subfloat{\includegraphics[width = 0.09\textwidth]{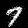}}
\subfloat{\includegraphics[width = 0.09\textwidth]{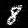}}
\subfloat{\includegraphics[width = 0.09\textwidth]{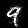}}\\
(d)MIDOT\\
\subfloat{\includegraphics[width = 0.09\textwidth]{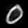}}
\subfloat{\includegraphics[width = 0.09\textwidth]{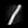}}
\subfloat{\includegraphics[width = 0.09\textwidth]{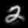}}
\subfloat{\includegraphics[width = 0.09\textwidth]{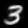}}
\subfloat{\includegraphics[width = 0.09\textwidth]{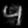}}
\subfloat{\includegraphics[width = 0.09\textwidth]{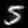}}
\subfloat{\includegraphics[width = 0.09\textwidth]{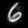}}
\subfloat{\includegraphics[width = 0.09\textwidth]{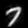}}
\subfloat{\includegraphics[width = 0.09\textwidth]{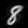}}
\subfloat{\includegraphics[width = 0.09\textwidth]{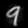}}\\
(e)Ours\\
\caption{Sample imputation results for Mnist at 0.3 missing rate: (a) Complete data images, (b) Initialized data images, (c) GAIN-imputed data images, (d) MIDOT-imputed data images, (e) The reconstruction data of proposed network. }
\label{imputation results} 
\end{figure*}

\begin{table*}[]
\caption{The dataset specific optimal values of the hyperparameter $\gamma$ a and the size of autoencoder.}
\label{experiment settings}
\resizebox{1\textwidth}{!}{
\begin{tabular}{|c||c|c|c|c|c|c|}
\toprule
Dataset      & Mnist        & Usps         & Fmnist       & Reuters      & COIL20       & Letter       \\ \hline
$\gamma$        & 100          & 100          & 100          & 150          & 150          & 150          \\ \hline
Encoder Size & 500/500/1000 & 500/500/1000 & 500/500/1000 & 500/500/2000 & 500/500/2000 & 500/500/2000 \\ \hline
Decoder Size & 1000/500/500 & 1000/500/500 & 1000/500/500 & 2000/500/500 & 2000/500/500 & 2000/500/500 \\ \bottomrule
\end{tabular}}
\end{table*}

\begin{table*}[]
\caption{The performance comparison (mean) of different algorithms on Mnist, where ’-’ means out of the GPU memory.}
\label{mnist}
\resizebox{\textwidth}{!}{
\begin{tabular}{|c|c||c|c|c|c|c|c|c|c|c|c|}
\toprule
\multicolumn{12}{|c|}{Mnist}\\
\hline
\multicolumn{2}{|c||}{Method}            & \multicolumn{5}{c|}{Shallow}          & \multicolumn{5}{c|}{Deep}                                \\ \hline
Missing Rate                              & Metrics               & MF    & ZF    & LRC   & MNC   & FSGR  & GAIN  & VAEAC & MIWAE & MDIOT & Ours \\
\hline
\multirow{3}{*}{10\%} & ACC  & 54.63 & 55.55 & 54.57 & 53.13 & 53.13 & 54.37 & 54.53 & -      & 55.09 & $\mathbf{92.16}$   \\
& NMI    & 48.97 & 49.13 & 49.08 & 48.24 & 47.87 & 48.90 & 49.36 &  -   & 49.16 & $\mathbf{85.77}$  \\     
& PUR     & 58.71 & 59.24 & 59.49 & 58.48 & 57.96 & 58.70 & 59.38 &  -   & 59.09 & $\mathbf{92.16}$    \\
\hline
\multirow{3}{*}{20\%}   & ACC   & 53.41 & 53.94 & 52.68 & 52.04 & 52.86 & 53.36 & 53.57 & - & 52.52 & $\mathbf{87.43}$   \\
 & NMI  & 48.65 & 48.49 & 47.20 & 47.34 & 47.41 & 48.70 & 48.58 &   -  & 48.83 & $\mathbf{82.98}$      \\
 & PUR & 58.88 & 58.65 & 57.99 & 58.48 & 58.00 & 59.31 & 58.31 &  -    & 58.71 & $\mathbf{87.43}$
 \\
\hline
\multirow{3}{*}{30\%}     & ACC    & 55.25 & 52.95 & 53.42 & 51.45 & 53.75 & 51.76 & 53.98 & -  & 53.09 & $\mathbf{86.74}$    \\
  & NMI   & 48.79 & 47.24 & 46.97 & 46.16 & 47.18 & 47.22 & 48.30 &   -   & 48.69 & $\mathbf{81.61}$     \\
  & PUR   & 59.06 & 57.48 & 58.08 & 56.85 & 58.06 & 57.23 & 58.11 &  -   & 58.63 & $\mathbf{86.74}$      \\
\hline
\multirow{3}{*}{40\%}   & ACC   & 54.88 & 55.87 & 52.82 & 51.72 & 52.25 & 53.16 & 55.38 &   -  & 53.15 & $\mathbf{85.66}$  \\
 & NMI  & 48.30 & 47.79 & 46.66 & 46.29 & 46.64 & 47.72 & 48.32 &   -  & 48.92 & $\mathbf{79.44}$       \\
 & PUR  & 58.91 & 58.98 & 58.10 & 57.43 & 58.16 & 57.81 & 58.75 &  -   & 58.84 & $\mathbf{85.66}$      \\
\hline
\multirow{3}{*}{50\%}     & ACC                   & 54.19 & 52.29 & 52.65 & 52.95 & 54.56 & 49.65 & 52.87 &  -  & 54.53 & $\mathbf{82.15}$                    \\
                                          & NMI                   & 47.27 & 45.51 & 46.26 & 46.35 & 46.87 & 45.50 & 47.15 &  -   & 50.66 & $\mathbf{74.94}$                    \\
                                          & PUR                   & 58.10 & 57.04 & 57.81 & 58.11 & 58.63 & 55.96 & 57.71 &   -   & 60.22 & $\mathbf{82.15}$                    \\
                                          \hline
\multirow{3}{*}{60\%}                     & ACC                   & 52.70 & 52.00 & 52.01 & 51.61 & 50.46 & 50.97 & 51.61 &  -  & 54.81 & $\mathbf{80.86}$                    \\
                                          & NMI                   & 45.87 & 42.25 & 44.78 & 44.85 & 44.34 & 45.06 & 46.07 &  -   & 50.52 & $\mathbf{71.25}$                    \\
                                          & PUR                   & 56.81 & 54.23 & 56.89 & 56.76 & 56.50 & 56.45 & 57.07 &   -  & 60.04 & $\mathbf{80.86}$                    \\
\hline
\multirow{3}{*}{70\%}                     & ACC                   & 52.32 & 50.33 & 50.36 & 49.89 & 49.83 & 47.75 & 50.27 &  -  & 54.88 & $\mathbf{75.15}$                    \\
                                          & NMI                   & 44.47 & 39.26 & 43.55 & 43.01 & 43.27 & 39.88 & 43.89 &  - & 49.09 & $\mathbf{61.92}$                    \\
                                          & PUR                   & 56.96 & 52.55 & 56.08 & 55.42 & 55.65 & 52.59 & 55.53 & -    & 59.74 & $\mathbf{75.15}$ \\
                                          \bottomrule
\end{tabular}}
\end{table*}

\begin{table*}[]
\caption{The performance comparison (mean) of different algorithms on Usps.}
\label{usps}
\resizebox{\textwidth}{!}{
\begin{tabular}{|c|c||c|c|c|c|c|c|c|c|c|c|}
\toprule
\multicolumn{12}{|c|}{Usps}\\
\hline
\multicolumn{2}{|c||}{Method}            & \multicolumn{5}{c|}{Shallow}          & \multicolumn{5}{c|}{Deep}                                \\ \hline
Missing Rate                              & Metrics               & MF    & ZF    & LRC   & MNC   & FSGR  & GAIN  & VAEAC & MIWAE & MDIOT & Ours \\
\hline
\multirow{3}{*}{10\%}            & ACC     & 63.96 & 63.63 & 62.70 & 66.14 & 62.15 & 63.91 & 64.69 & 65.09 & 62.49 & $\mathbf{77.5}$                     \\
                                 & NMI     & 61.28 & 61.00 & 60.62 & 61.71 & 60.87 & 61.22 & 61.98 & 62.18 & 61.61 & $\mathbf{77.35}$                     \\
                                 & PUR     & 71.69 & 70.85 & 70.97 & 73.12 & 69.82 & 71.46 & 72.20 & 72.79 & 71.08 & $\mathbf{84.34}$                     \\
                                 \hline
\multirow{3}{*}{20\%}            & ACC     & 64.04 & 63.23 & 64.55 & 63.63 & 64.01 & 64.33 & 63.54 & 59.12 & 65.86 & $\mathbf{77.5}$                     \\
                                 & NMI     & 60.58 & 60.31 & 60.15 & 60.49 & 60.30 & 61.87 & 61.63 & 54.80 & 63.49 & $\mathbf{77.67}$                     \\
                                 & PUR     & 70.68 & 70.51 & 71.75 & 71.21 & 71.13 & 71.96 & 71.73 & 62.40 & 73.64 & $\mathbf{83.9}$                      \\
                                 \hline
\multirow{3}{*}{30\%}            & ACC     & 62.99 & 62.83 & 64.55 & 60.25 & 62.60 & 62.11 & 64.63 & 60.49 & 63.52 & $\mathbf{77.65}$                     \\
                                 & NMI     & 60.31 & 59.65 & 59.93 & 58.85 & 58.90 & 60.26 & 61.07 & 59.56 & 62.67 & $\mathbf{75.64}$                     \\
                                 & PUR     & 70.66 & 70.32 & 71.42 & 68.52 & 69.75 & 69.54 & 71.89 & 67.16 & 71.06 & $\mathbf{83.11}$                     \\
                                 \hline
\multirow{3}{*}{40\%}            & ACC     & 59.60 & 61.08 & 61.64 & 63.94 & 61.52 & 62.08 & 62.48 & 53.26 & 64.02 & $\mathbf{75.54}$                     \\
                                 & NMI     & 58.60 & 57.90 & 57.58 & 59.00 & 57.95 & 59.44 & 59.22 & 49.61 & 62.46 & $\mathbf{73.63}$                     \\
                                 & PUR     & 68.54 & 69.20 & 68.92 & 70.75 & 69.33 & 69.24 & 69.41 & 56.68 & 70.61 & $\mathbf{81.86}$                     \\
                                 \hline
\multirow{3}{*}{50\%}            & ACC     & 60.55 & 60.66 & 61.32 & 60.81 & 60.53 & 59.61 & 62.73 & 48.96 & 63.73 & $\mathbf{73.77}$                     \\
                                 & NMI     & 58.24 & 55.44 & 56.50 & 57.46 & 55.54 & 58.14 & 58.49 & 44.48 & 63.02 & $\mathbf{70.52}$                     \\
                                 & PUR     & 69.18 & 67.85 & 68.13 & 68.32 & 67.15 & 67.33 & 70.24 & 51.67 & 70.89 & $\mathbf{80}$                        \\
                                 \hline
\multirow{3}{*}{60\%}            & ACC     & 61.77 & 56.24 & 62.58 & 63.49 & 56.21 & 57.99 & 60.15 & 57.00 & 63.94 & $\mathbf{73.11}$                     \\
                                 & NMI     & 57.68 & 50.32 & 56.12 & 57.23 & 52.49 & 54.19 & 56.02 & 53.54 & 63.02 & $\mathbf{66.74}$                     \\
                                 & PUR     & 69.07 & 62.59 & 69.45 & 69.36 & 64.03 & 64.94 & 67.77 & 60.15 & 71.29 & $\mathbf{77.64}$                     \\
                                 \hline
\multirow{3}{*}{70\%}            & ACC     & 58.76 & 51.42 & 58.16 & 58.61 & 52.29 & 53.70 & 56.75 & 55.81 & 61.95 & $\mathbf{68.48}$                     \\
                                 & NMI     & 53.41 & 43.77 & 51.34 & 52.97 & 46.88 & 49.52 & 51.85 & 51.26 & $\mathbf{60.35}$  & 59.95                    \\
                                 & PUR     & 65.55 & 57.14 & 65.03 & 66.04 & 59.65 & 60.28 & 64.11 & 61.70 & 69.98 & $\mathbf{70.5}$                     \\
                                 \bottomrule
\end{tabular}}
\end{table*}

\begin{table*}[]
\caption{The performance comparison (mean) of different algorithms on Fmnist, where ’-’ means out of the GPU memory.}
\label{fmnist}
\resizebox{\textwidth}{!}{
\begin{tabular}{|c|c||c|c|c|c|c|c|c|c|c|c|}
\toprule
\multicolumn{12}{|c|}{Fmnist}\\
\hline
\multicolumn{2}{|c||}{Method}            & \multicolumn{5}{c|}{Shallow}          & \multicolumn{5}{c|}{Deep}                                \\ \hline
Missing Rate                              & Metrics               & MF    & ZF    & LRC   & MNC   & FSGR  & GAIN  & VAEAC & MIWAE & MDIOT & Ours \\
\hline
\multirow{3}{*}{10\%}            & ACC     & 52.28 & 51.07 & 52.46 & 51.95 & 49.56 & 51.01 & 53.32 & -  & 51.55 & $\mathbf{63.79}$                     \\
                                 & NMI     & 50.65 & 50.04 & 51.09 & 50.41 & 50.11 & 50.83 & 51.07 & -  & 50.27 & $\mathbf{64.04}$                     \\
                                 & PUR     & 56.81 & 56.70 & 57.27 & 56.88 & 55.74 & 57.05 & 57.66 & -  & 56.96 & $\mathbf{68.28}$                     \\
                                 \hline
\multirow{3}{*}{20\%}            & ACC     & 51.81 & 52.31 & 51.70 & 49.93 & 50.31 & 51.80 & 50.99 & -  & 55.98 & $\mathbf{63.03}$                     \\
                                 & NMI     & 50.46 & 49.61 & 50.18 & 50.09 & 49.47 & 50.70 & 50.51 & -  & 50.53 & $\mathbf{63.28}$                     \\
                                 & PUR     & 57.09 & 56.18 & 56.69 & 56.42 & 55.86 & 56.63 & 56.29 & -  & 58.45 & $\mathbf{67.45}$                    \\
                                 \hline
\multirow{3}{*}{30\%}            & ACC     & 53.11 & 51.75 & 50.31 & 54.34 & 50.16 & 53.61 & 50.26 & -  & 51.38 & $\mathbf{59.03}$                     \\
                                 & NMI     & 50.97 & 49.55 & 50.02 & 50.67 & 49.83 & 51.41 & 49.99 & -  & 49.84 & $\mathbf{62.26}$                     \\
                                 & PUR     & 57.56 & 56.06 & 56.42 & 58.11 & 56.43 & 57.42 & 56.37 & -  & 56.55 & $\mathbf{63.52}$                     \\
                                 \hline
\multirow{3}{*}{40\%}            & ACC     & 53.04 & 50.32 & 52.72 & 52.72 & 51.25 & 50.88 & 53.94 & -  & 50.62 & $\mathbf{61.67}$                     \\
                                 & NMI     & 49.98 & 49.16 & 51.09 & 50.10 & 50.53 & 50.77 & 50.37 & -  & 49.38 & $\mathbf{62.32}$                    \\
                                 & PUR     & 56.73 & 55.68 & 57.95 & 57.46 & 56.90 & 56.31 & 57.59 & -  & 56.03 & $\mathbf{66.06}$                     \\
                                 \hline
\multirow{3}{*}{50\%}            & ACC     & 50.58 & 51.14 & 52.48 & 52.28 & 51.53 & 50.72 & 52.90 & -  & 51.84 & $\mathbf{59.5}$                      \\
                                 & NMI     & 50.13 & 48.88 & 50.83 & 50.50 & 49.78 & 49.88 & 50.27 & -  & 49.72 & $\mathbf{62.71}$                     \\
                                 & PUR     & 56.17 & 54.68 & 57.74 & 57.34 & 56.65 & 55.70 & 57.35 & -  & 56.34 & $\mathbf{64.48}$                     \\
                                 \hline
\multirow{3}{*}{60\%}            & ACC     & 50.36 & 44.51 & 52.59 & 51.88 & 53.73 & 51.25 & 51.21 & -  & 51.95 & $\mathbf{54.5}$                      \\
                                 & NMI     & 49.51 & 42.57 & 50.33 & 49.66 & 50.05 & 50.79 & 49.12 & -  & 49.88 & $\mathbf{58.57}$                     \\
                                 & PUR     & 55.71 & 47.26 & 57.17 & 56.65 & 57.44 & 56.28 & 56.20 & -  & 56.27 & $\mathbf{59.4}$                      \\
                                 \hline
\multirow{3}{*}{70\%}            & ACC     & 46.72 & 41.76 & 51.30 & $\mathbf{53.78}$  & 49.43 & 49.34 & 52.17 & -  & 51.60 & 50.61                    \\
                                 & NMI     & 48.21 & 38.28 & 49.24 & 50.67 & 48.68 & 50.82 & 49.42 & -  & 48.54 & $\mathbf{55.7}$                      \\
                                 & PUR     & 52.28 & 43.94 & 55.89 & $\mathbf{57.76}$  & 55.55 & 54.32 & 56.85 & -  & 55.81 & 55.11    \\               
                                 \bottomrule
\end{tabular}}
\end{table*}

\begin{table*}[]
\caption{The performance comparison (mean) of different algorithms on Reuters.}
\label{Reuters}
\resizebox{\textwidth}{!}{
\begin{tabular}{|c|c||c|c|c|c|c|c|c|c|c|c|}
\toprule
\multicolumn{12}{|c|}{Reuters}\\
\hline
\multicolumn{2}{|c||}{Method}            & \multicolumn{5}{c|}{Shallow}          & \multicolumn{5}{c|}{Deep}                                \\ \hline
Missing Rate                              & Metrics               & MF    & ZF    & LRC   & MNC   & FSGR  & GAIN  & VAEAC & MIWAE & MDIOT & Ours \\
\hline
\multirow{3}{*}{10\%}            & ACC     & 54.42 & 50.22 & 54.84 & 59.12 & 50.20 & 54.18 & 66.01 & 58.43 & 56.45 & $\mathbf{77.7}$                      \\
                                 & NMI     & 26.44 & 24.83 & 29.55 & 32.78 & 27.06 & 30.24 & 38.17 & 32.80 & 33.02 & $\mathbf{47.93}$                    \\
                                 & PUR     & 79.13 & 76.61 & 79.29 & $\mathbf{84.46}$  & 73.98 & 78.46 & 81.17 & 77.95 & 79.56 & 77.98                    \\
                                 \hline
\multirow{3}{*}{20\%}            & ACC     & 53.14 & 59.05 & 53.26 & 56.63 & 50.44 & 60.33 & 59.92 & 55.71 & 56.96 & $\mathbf{77.7}$                      \\
                                 & NMI     & 28.17 & 33.63 & 29.46 & 31.38 & 25.28 & 33.57 & 34.93 & 29.03 & 32.44 & $\mathbf{47.9}$                      \\
                                 & PUR     & 77.25 & 78.25 & 79.93 & 80.57 & 76.27 & 77.10 & $\mathbf{80.87}$  & 75.61 & 78.87 & 78.14                    \\
                                 \hline
\multirow{3}{*}{30\%}            & ACC     & 62.01 & 51.13 & 54.70 & 57.28 & 56.09 & 55.89 & 56.51 & 66.75 & 54.74 & $\mathbf{77.23}$                    \\
                                 & NMI     & 35.11 & 26.12 & 29.44 & 30.51 & 27.94 & 26.04 & 30.36 & 39.30 & 26.81 & $\mathbf{47.74}$                     \\
                                 & PUR     & 78.76 & 74.97 & 78.97 & $\mathbf{80.97}$  & 80.46 & 75.37 & 79.40 & 77.21 & 74.82 & 77.69                    \\
                                 \hline
\multirow{3}{*}{40\%}            & ACC     & 50.96 & 48.91 & 52.17 & 55.06 & 53.73 & 50.26 & 57.72 & 54.49 & 56.09 & $\mathbf{77.18}$                     \\
                                 & NMI     & 21.12 & 16.63 & 26.58 & 25.21 & 24.70 & 20.08 & 28.95 & 26.09 & 27.39 & $\mathbf{47.26}$                     \\
                                 & PUR     & 72.09 & 70.82 & 78.57 & 77.85 & 78.70 & 72.87 & $\mathbf{78.96}$  & 77.30 & 75.89 & 77.21                    \\
                                 \hline
\multirow{3}{*}{50\%}            & ACC     & 55.84 & 52.30 & 55.16 & 52.43 & 53.12 & 57.10 & 60.13 & 49.54 & 51.08 & $\mathbf{75.53}$                     \\
                                 & NMI     & 26.23 & 19.95 & 26.82 & 25.09 & 25.96 & 26.13 & 33.13 & 19.71 & 19.95 & $\mathbf{44.92}$                     \\
                                 & PUR     & 75.20 & 74.20 & 76.50 & 77.22 & 77.54 & 72.47 & $\mathbf{77.59}$  & 70.13 & 71.52 & 75.71                    \\
                                 \hline
\multirow{3}{*}{60\%}            & ACC     & 48.79 & 47.71 & 51.43 & 53.42 & 53.53 & 45.06 & 68.67 & 46.71 & 52.87 & $\mathbf{72.89}$                     \\
                                 & NMI     & 14.51 & 12.34 & 21.21 & 25.61 & 25.61 & 11.27 & 37.41 & 14.16 & 16.12 & $\mathbf{42.53}$                    \\
                                 & PUR     & 70.50 & 69.94 & 75.01 & 79.22 & 78.80 & 67.96 & $\mathbf{80.21}$  & 70.25 & 69.90 & 73.3                     \\
                                 \hline
\multirow{3}{*}{70\%}            & ACC     & 43.62 & 44.04 & 51.63 & 50.57 & 55.51 & 43.08 & 61.54 & 44.53 & 48.94 & $\mathbf{74.18}$                     \\
                                 & NMI     & 7.20  & 7.43  & 22.95 & 22.88 & 25.36 & 6.43  & 31.32 & 8.17  & 13.76 & $\mathbf{36.19}$                     \\
                                 & PUR     & 67.31 & 69.06 & 76.41 & 75.53 & $\mathbf{78.01}$  & 68.17 & 75.86 & 68.23 & 67.65 & 70.01                \\               
                                 \bottomrule
\end{tabular}}
\end{table*}

\begin{table*}[]
\caption{The performance comparison (mean) of different algorithms on COIL20.}
\label{coil20}
\resizebox{\textwidth}{!}{
\begin{tabular}{|c|c||c|c|c|c|c|c|c|c|c|c|}
\toprule
\multicolumn{12}{|c|}{COIL20}\\
\hline
\multicolumn{2}{|c||}{Method}            & \multicolumn{5}{c|}{Shallow}          & \multicolumn{5}{c|}{Deep}                                \\ \hline
Missing Rate                              & Metrics               & MF    & ZF    & LRC   & MNC   & FSGR  & GAIN  & VAEAC & MIWAE & MDIOT & Ours \\
\hline
\multirow{3}{*}{10\%}            & ACC     & 58.03 & 57.26 & 56.25 & 58.10 & 59.19 & 58.13 & 57.99 & 56.60 & 58.61 & $\mathbf{71.58}$                     \\
                                 & NMI     & 73.16 & 71.52 & 72.52 & 73.82 & 74.00 & 73.94 & 73.59 & 72.99 & 73.09 & $\mathbf{79.97}$                     \\
                                 & PUR     & 63.01 & 62.64 & 61.12 & 63.39 & 63.27 & 62.63 & 63.14 & 61.49 & 62.88 & $\mathbf{73.96}$                     \\
                                 \hline
\multirow{3}{*}{20\%}            & ACC     & 58.56 & 57.21 & 58.85 & 57.21 & 55.69 & 58.30 & 55.63 & 57.28 & 59.51 & $\mathbf{71.94}$                     \\
                                 & NMI     & 73.32 & 71.31 & 73.91 & 73.24 & 71.96 & 73.11 & 72.64 & 73.55 & 73.89 & $\mathbf{80.78}$                     \\
                                 & PUR     & 63.03 & 62.39 & 64.06 & 62.91 & 60.38 & 62.35 & 60.44 & 62.01 & 63.99 & $\mathbf{74.93}$                     \\
                                 \hline
\multirow{3}{*}{30\%}            & ACC     & 57.62 & 55.13 & 60.90 & 60.09 & 59.20 & 53.03 & 59.33 & 58.15 & 57.29 & $\mathbf{71.67}$                    \\
                                 & NMI     & 72.04 & 69.28 & 75.28 & 74.28 & 73.21 & 71.14 & 74.29 & 72.17 & 72.13 & $\mathbf{80.11}$                     \\
                                 & PUR     & 61.65 & 59.65 & 65.51 & 64.35 & 63.16 & 58.73 & 64.44 & 62.22 & 61.91 & $\mathbf{74.24}$                     \\
                                 \hline
\multirow{3}{*}{40\%}            & ACC     & 55.89 & 48.82 & 60.67 & 59.42 & 57.74 & 56.25 & 59.40 & 59.33 & 57.72 & $\mathbf{61.18}$                     \\
                                 & NMI     & 70.31 & 62.11 & 74.57 & 74.51 & 71.57 & 72.30 & 73.86 & 72.84 & 72.68 & $\mathbf{77.23}$                     \\
                                 & PUR     & 59.60 & 50.96 & 64.13 & 63.99 & 62.29 & 61.08 & 63.77 & 63.07 & 62.33 & $\mathbf{75.21}$                     \\
                                 \hline
\multirow{3}{*}{50\%}            & ACC     & 57.05 & 37.79 & 60.74 & 60.23 & 56.53 & 54.84 & 57.38 & 54.98 & 57.67 & $\mathbf{64.51}$                     \\
                                 & NMI     & 69.83 & 51.89 & 75.09 & 74.88 & 70.84 & 71.07 & 72.73 & 67.72 & 73.18 & $\mathbf{76.53}$                     \\
                                 & PUR     & 60.40 & 39.17 & 65.14 & 64.54 & 60.88 & 59.29 & 62.54 & 57.08 & 62.88 & $\mathbf{67.36}$                    \\
                                 \hline
\multirow{3}{*}{60\%}            & ACC     & 52.08 & 22.10 & 60.78 & 64.35 & 57.03 & 52.00 & 58.16 & 56.32 & 58.83 & $\mathbf{63.19}$                     \\
                                 & NMI     & 65.95 & 35.31 & 74.72 & 76.08 & 69.71 & 65.28 & 72.62 & 69.58 & 73.69 & $\mathbf{76.14}$                    \\
                                 & PUR     & 55.22 & 22.61 & 65.22 & $\mathbf{68.43}$  & 61.02 & 54.99 & 62.56 & 58.45 & 63.40 & 67.36                    \\
                                 \hline
\multirow{3}{*}{70\%}            & ACC     & 43.41 & 15.71 & 58.33 & $\mathbf{62.19}$  & 50.13 & 44.51 & 57.51 & 53.50 & 59.52 & 60.76                    \\
                                 & NMI     & 57.13 & 24.58 & 72.32 & $\mathbf{74.37}$  & 61.95 & 53.63 & 71.95 & 68.26 & 73.98 & 71.21                    \\
                                 & PUR     & 44.79 & 16.19 & 63.00 & $\mathbf{66.48}$  & 52.59 & 45.91 & 61.74 & 56.11 & 64.49 & 62.99                        \\               
                                 \bottomrule
\end{tabular}}
\end{table*}

\begin{table*}[!htbp]
\caption{The performance comparison (mean) of different algorithms on Letter.}
\label{letter}
\resizebox{\textwidth}{!}{
\begin{tabular}{|c|c||c|c|c|c|c|c|c|c|c|c|}
\toprule
\multicolumn{12}{|c|}{Letter}\\
\hline
\multicolumn{2}{|c||}{Method}            & \multicolumn{5}{c|}{Shallow}          & \multicolumn{5}{c|}{Deep}                                \\ \hline
Missing Rate                              & Metrics               & MF    & ZF    & LRC   & MNC   & FSGR  & GAIN  & VAEAC & MIWAE & MDIOT & Ours \\
\hline
\multirow{3}{*}{10\%}            & ACC     & 36.72 & 36.38 & 37.38 & 36.91 & 36.73 & 37.24 & 36.99 & 36.56 & 36.39 & $\mathbf{51.12}$                     \\
                                 & NMI     & 39.47 & 39.00 & 39.59 & 39.53 & 39.56 & 39.51 & 39.61 & 38.86 & 39.18 & $\mathbf{54.82}$                     \\
                                 & PUR     & 39.27 & 39.00 & 39.91 & 39.68 & 39.08 & 39.68 & 39.55 & 38.69 & 39.05 & $\mathbf{52.95}$                    \\
                                 \hline
\multirow{3}{*}{20\%}            & ACC     & 36.01 & 36.19 & 36.16 & 35.29 & 36.98 & 36.91 & 36.77 & 33.91 & 37.26 & $\mathbf{48.5}$                      \\
                                 & NMI     & 38.95 & 38.52 & 39.24 & 38.10 & 39.49 & 39.24 & 39.54 & 36.73 & 39.01 & $\mathbf{54.83}$                     \\
                                 & PUR     & 38.59 & 38.62 & 38.89 & 37.74 & 39.60 & 39.18 & 39.26 & 36.22 & 39.26 & $\mathbf{51.26}$                     \\
                                 \hline
\multirow{3}{*}{30\%}            & ACC     & 35.96 & 35.94 & 37.27 & 34.99 & 36.77 & 36.59 & 36.87 & 30.18 & 36.01 & $\mathbf{52.14}$                     \\
                                 & NMI     & 38.74 & 37.91 & 39.66 & 37.81 & 39.71 & 39.17 & 39.45 & 33.16 & 38.63 & $\mathbf{56.03}$                     \\
                                 & PUR     & 38.62 & 38.19 & 39.90 & 37.53 & 39.43 & 39.18 & 39.54 & 32.16 & 38.57 & $\mathbf{54.68}$                     \\
                                 \hline
\multirow{3}{*}{40\%}            & ACC     & 35.65 & 35.13 & 36.48 & 34.82 & 37.00 & 36.69 & 36.49 & 26.83 & 35.86 & $\mathbf{49.93}$                     \\
                                 & NMI     & 38.18 & 36.88 & 39.07 & 37.06 & 39.46 & 38.87 & 38.94 & 29.64 & 38.01 & $\mathbf{55.26}$                     \\
                                 & PUR     & 38.12 & 37.54 & 39.11 & 37.22 & 39.53 & 39.20 & 39.04 & 28.12 & 38.22 & $\mathbf{52.65}$                    \\
                                 \hline
\multirow{3}{*}{50\%}            & ACC     & 35.67 & 33.83 & 36.98 & 35.08 & 37.48 & 36.03 & 36.51 & 23.28 & 35.11 & $\mathbf{48.32}$                     \\
                                 & NMI     & 38.02 & 34.82 & 39.34 & 35.51 & 39.50 & 38.10 & 38.90 & 26.34 & 37.59 & $\mathbf{53.05}$                     \\
                                 & PUR     & 37.98 & 35.54 & 39.48 & 36.50 & 39.66 & 38.45 & 38.89 & 24.55 & 37.41 & $\mathbf{50.54}$                    \\
                                 \hline
\multirow{3}{*}{60\%}            & ACC     & 35.38 & 31.15 & 36.99 & 31.82 & 37.51 & 34.57 & 35.75 & 22.51 & 35.25 & $\mathbf{43.47}$                     \\
                                 & NMI     & 36.41 & 30.72 & 38.87 & 31.17 & 39.17 & 36.70 & 38.22 & 23.98 & 36.66 & $\mathbf{49.02}$                     \\
                                 & PUR     & 37.37 & 32.00 & 39.24 & 32.49 & 39.83 & 37.01 & 38.47 & 23.79 & 37.20 & $\mathbf{45.88}$                     \\
                                 \hline
\multirow{3}{*}{70\%}            & ACC     & 32.46 & 25.04 & $\mathbf{37.18}$  & 27.34 & 37.11 & 32.27 & 35.31 & 20.14 & 33.07 & 36.58                    \\
                                 & NMI     & 33.19 & 24.47 & 38.36 & 26.32 & 38.59 & 34.35 & 36.93 & 22.03 & 33.85 & $\mathbf{41.62}$                     \\
                                 & PUR     & 34.15 & 25.45 & 39.54 & 27.76 & $\mathbf{39.38}$  & 34.99 & 37.62 & 20.80 & 34.46 & 38.97                    \\               
                                 \bottomrule
\end{tabular}}
\end{table*}

\end{document}